\documentclass[10pt,journal,compsoc]{IEEEtran}

\usepackage{array}
\usepackage{times}
\usepackage{epsfig}
\usepackage{graphicx}
\usepackage{amsmath,amsthm}
\usepackage{amssymb,bbm,bm}
\usepackage[mathscr]{eucal}
\usepackage{amsfonts}
\usepackage{subfigure,cases,multirow}
\usepackage[ruled,vlined,boxed]{algorithm2e}
\usepackage[colorlinks,linkcolor=blue,citecolor=blue,urlcolor=blue]{hyperref}
\usepackage[compress]{cite}
\usepackage{diagbox}

\usepackage{todonotes}
\usepackage[capitalize]{cleveref} 
\theoremstyle{plain}
\newtheorem{definition}{Definition}

\newtheorem{proposition}{Proposition}

\newcommand{\bE}{\mathbb E}
\newcommand{\bS}{\mathbb S}
\newcommand{\bR}{\mathbb R}
\newcommand{\bZ}{\mathbb Z}
\newcommand{\bM}{\mathbb M}

\newcommand{\kL}{\mathfrak L}
\newcommand{\kC}{\mathfrak C}
\newcommand{\kF}{\mathfrak F}
\newcommand{\kH}{\mathfrak H}
\newcommand{\ku}{\mathfrak u}
\newcommand{\kR}{\mathfrak R}

\newcommand{\sC}{\mathscr C}

\newcommand{\cS}{\mathcal S}
\newcommand{\cU}{\mathcal U}

\newcommand{\cG}{\mathcal G}
\newcommand{\cH}{\mathcal H}
\newcommand{\cL}{\mathcal L}
\newcommand{\cC}{\mathcal C}
\newcommand{\cF}{\mathcal F}
\newcommand{\cB}{\mathcal B}
\newcommand{\cO}{\mathcal O}
\newcommand{\cD}{\mathcal D}
\newcommand{\cM}{\mathcal M}

\newcommand\rEM{{\rm EM}}
\newcommand\rmD{{\rm D}}
\newcommand\rRS{{\rm RS}}

\def\<{\langle} 
\def\>{\rangle}
\def\sm{\setminus}
\newcommand{\ts}{\textstyle}
\newcommand{\vp}{\varphi}
\newcommand{\ve}{\varepsilon}
\def\ray{\Re_z(p)}

\newcommand{\fp}{\mathbf p}
\newcommand{\fx}{\mathbf x}
\newcommand{\fV}{\mathbf V}

\newcommand{\fy}{\mathbf y}

\newcommand{\dx}{\dot x}
\newcommand{\dfq}{{\dot{\mathbf q}}}
\newcommand{\dfnt}{{\dot{\mathrm n}}_\theta} % \mathbf is for 3D vectors
\newcommand{\dfn}{{\dot{\mathrm n}}}

\newcommand{\dfv}{{\dot{\mathbf v}}}
\newcommand{\dfe}{\dot{\mathbf e}}

\def\lag{\mathscr{L}} 

\DeclareMathOperator\tcurv{\mathcal{K}}
\DeclareMathOperator\curl{curl}

\DeclareMathOperator\Lip{Lip}
\DeclareMathOperator\JS{JS}
\DeclareMathOperator\RS{SR}
\DeclareMathOperator\RG{GR}

\makeatletter
\newcommand{\oset}[3][0ex]{%
  \mathrel{\mathop{#3}\limits^{
    \vbox to#1{\kern-2\ex@
    \hbox{$\scriptstyle#2$}\vss}}}}
\makeatother

\newlength\savewidth
\newcommand\shline{\noalign{\global\savewidth\arrayrulewidth\global\arrayrulewidth 1.0pt}\hline\noalign{\global\arrayrulewidth\savewidth}}

\newlength\savedwidth

\ifCLASSINFOpdf
\else
\fi

\begin{document}

\title{Geodesic Models with Convexity Shape Prior}
\author{Da Chen, Jean-Marie Mirebeau, Minglei Shu, Xuecheng Tai and~Laurent D. Cohen,~\IEEEmembership{Fellow,~IEEE}

\IEEEcompsocitemizethanks{
\IEEEcompsocthanksitem Da Chen and Minglei Shu are with Qilu University of Technology (Shandong Academy of Sciences),  Shandong Artificial Intelligence Institute, Jinan, China. e-mail: dachen.cn@hotmail.com, shuml@sdas.org\protect %\\
\IEEEcompsocthanksitem Jean-Marie Mirebeau is with Department of Mathematics, Centre Borelli, ENS Paris-Saclay, CNRS, University Paris-Saclay, 91190, Gif-sur-Yvette, France. e-mail: jean-marie.mirebeau@ens-paris-saclay.fr
\protect%\\
\IEEEcompsocthanksitem Xue-Cheng Tai is with Hong Kong Centre for Cerebro-cardiovascular Health Engineering Building 19W, Hong Kong Science Park, Shatin, N.T., Hong Kong.  e-mail: xuechengtai@gmail.com  \protect%\\
\IEEEcompsocthanksitem Laurent D. Cohen is with University Paris Dauphine, PSL Research University, CNRS, UMR 7534, CEREMADE, 75016 Paris, France. e-mail: cohen@ceremade.dauphine.fr\protect\\
%Corresponding author: Minglei Shu.
}% <-this % stops an unwanted space
\thanks{}
}

\markboth{Journal of \LaTeX\ Class Files,~Vol.~14, No.~8, August~2021}%
{Shell \MakeLowercase{\textit{et al.}}: Bare Demo of IEEEtran.cls for Computer Society Journals}

\IEEEtitleabstractindextext{
\begin{abstract}
The minimal geodesic models established upon the eikonal equation framework are  capable of finding suitable solutions in various image segmentation scenarios. Existing  geodesic-based segmentation approaches usually exploit  image features in conjunction with geometric regularization terms, such as Euclidean curve length or curvature-penalized length, for computing geodesic curves. In this paper, we take into account a more complicated problem: finding curvature-penalized geodesic paths with a convexity shape prior. We establish new geodesic models relying on the strategy of orientation-lifting, by which a planar curve can be mapped to an high-dimensional orientation-dependent space. The convexity shape prior serves as a constraint for the construction of local geodesic metrics encoding a particular curvature constraint. Then the geodesic distances and the corresponding closed geodesic paths  in the orientation-lifted space can be efficiently computed through state-of-the-art Hamiltonian fast marching method. In addition, we apply the proposed geodesic models to the active contours, leading to efficient interactive image segmentation algorithms that preserve the advantages of convexity shape prior and curvature penalization.
\end{abstract}

\begin{IEEEkeywords}
Geodesic, convexity shape prior, curvature penalization, eikonal equation, fast marching method, image segmentation.
\end{IEEEkeywords}}

\maketitle
\IEEEdisplaynontitleabstractindextext

\IEEEpeerreviewmaketitle

\IEEEraisesectionheading{\section{Introduction}
\label{sec:introduction}}
\IEEEPARstart{i}{mage} segmentation is a fundamental and challenging problem in many fields of image analysis, computer vision and medical imaging. Segmentation approaches based on the energy minimization frameworks are able to cope with a great amount of complicated segmentation tasks. As important advantages, these segmentation models can benefit from significant flexibility in the accommodation of various image appearance features, and particularly in the use of efficient shape priors for image segmentation.  

As considered in many variational segmentation models, a typical  energy functional usually contains an image appearance term such as the region-based homogeneity measure, and a geometry regularization term which in essence implicitly defines shape priors for the targets.
 The euclidean  curve length, which usually taken as a first-order regularizer, has proven its efficiency in suppressing ambiguous segmentations, thus extensively exploited in either active contour models~\cite{chan2001active,chan2006algorithms,cremers2007review} or graph optimization-based models~\cite{boykov2006grapy,couprie2011power}.  The second-order curvature penalization is regarded as a significant variant of the curve length-based regularization, which is capable of producing smooth segmentation contours~\cite{schoenemann2012linear,el2016contrast,zhu2013image}. However, it is sometimes insufficient to search for favorable segmentation results exploiting only geometric regularities as shape priors, especially when handling images with complicated gray level distribution or  edges of weak visibility. In contrast, the strategy of incorporating shape-driven priors into the energy functionals is able to yield more strong and efficient constraints for image segmentation, thus can increase the segmentation accuracy. These shape priors are often carried out via a statistical model associated with the target shapes or contours~\cite{prevost2014tagged,bresson2006variational,cremers2007review,chan2005level,cremers2008shape}. The implementation of the statistical shape priors can encourage satisfactory segmentations, even in the absence of reliable region-based appearance features and edge saliency features used for distinguishing distinct subregions.

\begin{figure*}[t]
\centering
\includegraphics[height=3.2cm]{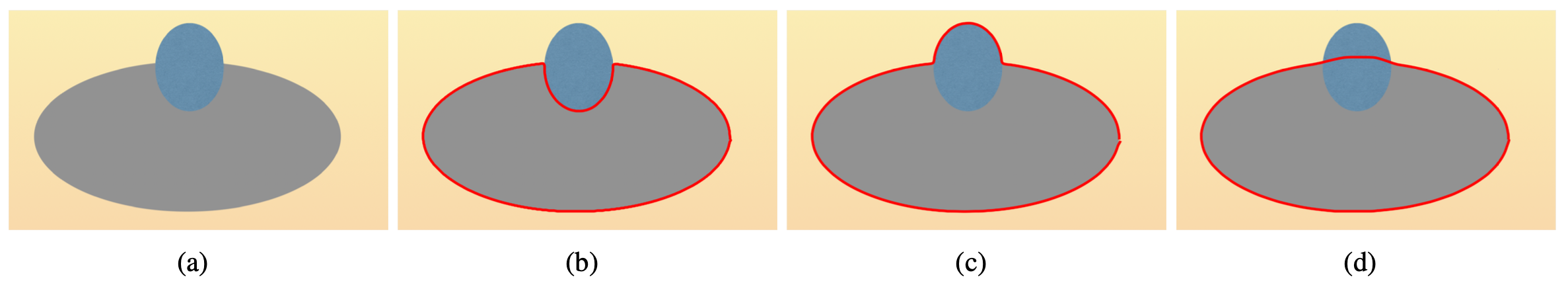}
\caption{A comparison example for geodesic curves from different models. \textbf{a} A synthetic image. \textbf{b} to \textbf{d}: Results from the region-based model~\cite{chen2016finsler}, the Euler-Mumford elastica model~\cite{chen2017global} and the proposed elastica model with the convexity shape prior.}
\label{fig_Demo}
\end{figure*} 

Recently, the  convexity and star convexity shape constraints were introduced as flexible priors. Existing image segmentation approaches in conjunction with these shape priors can be loosely categorized as either discrete or continuous types. In the discrete setting, the convexity prior~\cite{gorelick2016convexity}, the star convexity prior~\cite{veksler2008star}, or geodesic star convexity~\cite{gulshan2010geodesic} are characterized as a regularization term to construct the discrete energy functionals together with image data-driven terms. The energy minimization can be addressed by the  efficient graph cut algorithm~\cite{boykov2006grapy}. In~\cite{royer2016convexity,gorelick2017multi}, the convexity prior was incorporated into graph-based segmentation framework to solve multi-region segmentation tasks. The hedgehog-like shape prior~\cite{isack2016hedgehog} was introduced to image segmentation as a practical generalization of the star convexity constraints~\cite{veksler2008star,gorelick2017multi}.  Isack~\emph{et al.}~\cite{isack2018kconvexity} proposed a flexible k-convexity prior-based segmentation model allowing overlaps between different regions. However, these graph-based approaches with shape prior constraints did not consider the curvature regularization. 

In the continuous setting, the convexity shape prior is usually exploited as a constraint in the active contour models~\cite{luo2019convex,shi2021convexity,yan2020convexity,bae2017augmented}. Among them,  the curvature property is taken as a crucial feature to characterize the convexity shape prior during the curve evolution. 
Specifically, in~\cite{luo2019convex,yan2020convexity}, the authors revealed the relationship between the convexity property of a region and the signed distance map associated to its boundary. They proved  that a region is convex if the Laplacian of its signed distance map, which approximates the curvature of its iso-contours,  has a constant sign within this region. In this approach, the convexity shape prior is used to build the search space for target regions. Shi and Li~\cite{shi2021convexity} introduced an alternative level set-based active contour model, in which the curvature values of the iso-contours of the level set function are leveraged to establish a convex variant of the curve evolution flow. 
An important shortcoming  for these models is that only the sign of the curvature is utilized for segmentation, while its magnitude is ignored. 
Bae~\emph{et~al.}~\cite{bae2017augmented} introduced a variant of Euler-Mumford elastica model, where the convexity constraint is implicitly  taken as the regularizer of an energy, measured using the absolute curvature. Unfortunately, finding convex contours by this model heavily relies on the minimization of the energy, thus requiring a demanding  numerical scheme. 

\subsection{Geodesic Models}
The snakes model~\cite{kass1988snakes} is referred to as one of the earliest variational models  leveraging continuous curves to  extract image boundary features. However, the sensitivity to local minima of the snakes energy functional and the difficulty in finding suitable numerical solutions prevent this model from practical applications.  In order to address these shortcomings, Cohen and Kimmel~\cite{cohen1997global} proposed an elegant minimal path model, or geodesic model, to globally minimize a weighted curve length in an eikonal equation framework. This original geodesic model has inspired a great variety of relevant approaches, due to its significant advantages in both global optimality and efficient numerical solvers. Among them, many  geodesic models have contributed to develop various local geodesic metrics~\cite{melonakos2008finsler,peyre2010geodesic}, which extend globally optimal paths to accommodate various scenarios of image analysis. For instances, the curvature-penalized models introduced in~\cite{chen2017global,duits2018optimal,mirebeau2018fast} took into account an idea of orientation-lifting to address the computation problem of geodesic curves with curvature penalization. Using a suitable relaxation scheme, the geodesic distances involving curvature penalization can be efficiently estimated through the Hamiltonian fast marching method~\cite{mirebeau2018fast,mirebeau2019hamiltonian}.

In the context of image segmentation, the basic objective for geodesic-based methods is to find simple closed curves as the descriptor for boundaries of interest. Specifically, image gradients are taken as crucial features for computing geodesic curves, as considered in~\cite{appleton2005globally,appia2011active,mille2015combination}. Moreover, Chen~\emph{et al.}~\cite{chen2016finsler} introduced a region-based Randers geodesic model for computing  closed geodesic curves which can encode regional homogeneity features, thus build the connection between the eikonal equation-based geodesic models and the region-based active contours. In~\cite{chen2021Geodesic}, the authors exploited the curvature-penalized geodesic curves for image segmentation, in which the region-based homogeneity features were implicitly encoded in the geodesic metrics. 
 
Despite great advances, existing geodesic approaches only utilize the geometric length-based shape priors as the regularization term in tracking geodesic paths. In order to overcome this drawback, we introduce three new geodesic models, as an extension of a conference paper~\cite{chen2021elastica}, such that the curvature regularization, the region-based homogeneity measure and the convexity shape prior  can be integrated for the construction of metrics, which to our best knowledge is original. Finally, we show a comparison example in Fig.~\ref{fig_Demo} for geodesic paths respectively derived from two state-of-the-art geodesic models  and from the proposed geodesic model imposed with convexity shape prior. In this experiment, we aim to extract an ellipse shape. One can see that only the proposed geodesic model can find suitable segmentation contour, illustrating the advantages of imposing convexity shape prior in computing geodesic paths.

\subsection{Contributions and Paper Outline}
The convexity of a simple closed curve is defined by the sign of its curvature. By this definition, we propose new geodesic models, featuring convexity shape prior, under the eikonal equation framework. Basically, the proposed models  involve the establishment of metrics and the construction of search space for admissible geodesic paths. In summary, the contributions are threefold: 

\noindent$-~$\emph{Geodesic metrics encoding curvature restriction}. We introduce three new local geodesic metrics, such that the curvature of the physical projections of the orientation-lifted geodesic paths associated to these metrics have a constant sign. In addition, we also discuss the Hamiltonians of the proposed geodesic metrics, for which the discrete approximations, in terms of scalar nonnegative weights and offsets with integer components, are leveraged for finding the numerical solutions to the eikonal equations.

\noindent$-~$\emph{Construction of search space for geodesic paths}. We define a search space for the minimization of the weighted curve length which is measured via the proposed local geodesic metrics. This search space is in essence a set collecting all admissible orientation-lifted curves whose physical projection curve is simple closed and convex.
%such that the physical projection curve of any candidate minimizer is supposed to be simple closed and convex.
%In particular, we  work with the tool of total curvature to ensure the simplicity requirement of the search space.

\noindent$-~$\emph{Applications in active contours}. We investigate the applications of the proposed geodesic models in active contour problem and in interactive image segmentation. As a consequence, the image segmentation procedure can blend the benefits from the convexity shape prior, curvature regularization and  efficient image features.

This paper is organized as follows. \Cref{sec_BG} gives the background on the curvature-penalized geodesic models  and their discretization schemes. Three geodesic models  imposed with a convexity shape prior are presented in \cref{sec_ConvexityTheorem,sec_SearchingSpace}. The numerical implementation is presented in \cref{sec_HFM}. In \cref{sec_ACAPP}, we show how to exploit the proposed geodesic models to the active contours and image segmentation.  The experimental results and the conclusion are respectively  presented in \cref{sec_Exp,sec_Conclusion}.

\section{Curvature Penalized Minimal Paths}
\label{sec_BG}
\textbf{Notations}.
Let  $\bM:=\Omega\times\bS^1$ be an orientation-lifted space, where $\Omega \subset \bR^2$ is a bounded domain, and  $\bS^1:=\bR/2\pi\bZ$ can be identified with $[0,2\pi[$ equipped with a periodic boundary condition. A point $\fx=(x,\theta)$ is a pair comprised of a \emph{physical} position $x\in \Omega$ and an \emph{angular} coordinate $\theta\in \bS^1$. 
The tangent space to $\bM$ is represented by $\bE := \bR^2\times\bR$, at any base point $\fx$, and its elements are denoted as $\dot\fx=(\dx,\dot\theta)$. 
In addition, we denote by $a_+ := \max \{0,a\}$ the positive part of a real number $a \in \bR$, and likewise $a_- := \max \{0,-a\}$. 
Finally, the conventions $0\times\infty=0$ and $a_+^2 := (a_+)^2$ are adopted in the remaining of this paper.

\subsection{Orientation lifting for curvature representation}
\label{sec_OL}
The proposed convexity-constrained geodesic models are obtained as variants of the classical curvature-penalized models~\cite{chen2017global,duits2018optimal,mirebeau2018fast}.
Their common foundation is to evaluate curvature using an orientation lifting~\cite{chambolle2019total}.
Consider a smooth curve $\gamma : [0,1] \to \Omega$, with non-vanishing velocity\footnote{The non-vanishing velocity assumption is implicit in the sequel.}. Then there exists a unique function $\eta : [0,1] \to \bS^1$ obeying for all $\varrho \in [0,1]$ 
\begin{equation}
\label{eq_TurningAngle}
\dot{\gamma}(\varrho)=\dfn_{\eta(\varrho)}\|\dot\gamma(\varrho)\|,
\end{equation}
where $\dfnt=(\cos\theta, \sin\theta)$ denotes the unit vector of angle $\theta$ w.r.t.\ the horizontal axis, and $\dot\gamma$ is the first-order derivative of $\gamma$. Thus,  $\eta(\varrho)$ encodes the tangent direction at $\gamma(\varrho)$.
By \eqref{eq_TurningAngle}, we define the orientation-lifted curve 
\begin{equation}
\label{eq_OL}
\Gamma:=(\gamma,\eta):\varrho\in[0,1]\mapsto\Gamma(\varrho)\in\bM, 
\end{equation}
whose first-order derivative reads $\dot\Gamma(\varrho)=(\dot\gamma(\varrho),\dot\eta(\varrho))\in\bE$. Conversely, we refer to $\gamma$ as the \emph{physical projection}, and $\eta$ as the \emph{angular component}, of the orientation lifted curve $\Gamma$.  
The curvature $\kappa:[0,1]\to\bR$ of the planar curve $\gamma$ is obtained as 
\begin{equation}
\label{eq_Curvature}
\kappa(\varrho):=\dot\eta(\varrho)/\|\dot\gamma(\varrho)\|. 
\end{equation}
%As a consequence, the curvature $\kappa$ is represented through the ratio of two first-order derivatives. 

\subsection{Curvature-penalized Geodesic Models}
\label{subsec_CurvatureGeodesic}
In this paper, we consider three curvature-penalized  minimal path models: the Reeds-Sheep forward (RSF) model~\cite{duits2018optimal}, the Dubins car model~\cite{mirebeau2018fast} and the Euler-Mumford (EM) elastica model~\cite{chen2017global}. 
The energy functionals defining these models involve a scalar-valued curvature penalty function $\cC:\bR\to]0,\infty]$, described in \cref{sec_ConvexityTheorem}. For a smooth curve $\gamma:[0,1]\to\Omega$,
with tangent direction $\eta$ and curvature $\kappa$, see Eqs.~\eqref{eq_TurningAngle} and \eqref{eq_Curvature}, the energy reads
\begin{equation}
\label{eq_SecOrderLength}
\int_0^1 \psi(\gamma(\varrho),\eta(\varrho))\,\cC(\beta\kappa(\varrho))\|\dot\gamma(\varrho)\|d\varrho,
\end{equation}
where $\psi:\bM\to\bR^+$ is a user-defined cost function, derived in this paper from the image data, see Section~\ref{sec_ACAPP}. 
The parameter $\beta\in\bR^+$ has the dimension of a radius of curvature, and modulates the strength of the curvature penalty. 

The curvature-penalized length~\eqref{eq_SecOrderLength} involves second-order derivatives of the curve $\gamma$, implicitly through  $\cC(\beta\kappa)$, and is thus not directly amenable to global optimization via the eikonal equation framework. 
Using the orientation lifting~\eqref{eq_OL} one can however express curvature as a ratio of first-order derivatives~\eqref{eq_Curvature}, which motivates the following equivalent definition of energy
\begin{equation}
\label{eq_FirstOrderLength}
\cL(\Gamma):=\int_0^1 \psi(\Gamma(\varrho)) \, \cF(\Gamma(\varrho),\dot\Gamma(\varrho))\, d\varrho,
\end{equation}
where $\cF:\bM\times\bE\to[0,\infty]$ is an orientation-lifted Finsler metric defined for any point $\fx=(x,\theta)\in\bM$ and any vector  $\dot\fx=(\dx,\dot\theta)\in\bE$. 
The geodesic metric can be expressed in terms of the curvature penalty function $\cC$ and modulation parameter $\beta$~\cite{mirebeau2018fast}
\begin{equation}
\label{eq_GeneralMetric}
\cF(\fx,\dot\fx)=
\begin{cases}
\cC(\beta\dot\theta/\|\dot x\|)\|\dot x\|,&\text{if~}\dx=\dfnt\|\dx\|,\\
\infty,&\text{otherwise},
\end{cases}	
\end{equation}
if $\dot x\neq 0$. (The lower semi-continuous limit is used if $\dot x=0$.) 

The equivalence between the functionals~\eqref{eq_SecOrderLength} and~\eqref{eq_FirstOrderLength} follows from the expression~\eqref{eq_Curvature} of the curvature $\kappa$. Let $\Lip([0,1],\bM)$ be the collection of all the orientation-lifted curves $\Gamma:[0,1]\to\bM$ with Lipchitz continuity. In order to compute the geodesic path from a source point $\fp\in \bM$ to a target point $\fx\in \bM$, we first define a geodesic distance map $\cU_\fp:\bM\to[0,\infty)$, also known as the minimal action map, as follows
\begin{equation}
\label{eq_GeoDistMap}
\cU_{\fp}(\fx)=\inf_{\Gamma\in\Lip([0,1],\bM)}\,\Big\{\cL(\Gamma);\, \Gamma(0)=\fp,\Gamma(1)=\fx\Big\}.
\end{equation}
As in~~\cite{sethian2003ordered,mirebeau2018fast}, this distance map is the unique viscosity solution to a generalized eikonal equation, or a static Hamiltonian-Jacobi equation, based on the Hamiltonian $\cH$ of the metric $\cF$:
\begin{equation}
\label{eq_HamiltonEikonal}	
\cH_\fx(d\cU_\fp(\fx))=\frac{1}{2}\psi(\fx)^2,\quad \forall\fx\in\bM\backslash\{\fp\},
\end{equation}
with $\cU_\fp(\fp)=0$, and with outflow boundary condition on $\partial \bM$, where $d\cU_\fp$ stands for the differential of the geodesic distance map $\cU_\fp$. The Hamiltonian $\cH$ is defined from the metric $\cF$ by Legendre-Fenchel duality, as follows
\begin{equation}
\label{eq_Hamiltonian}
\cH_\fx(\hat\fx):= \sup_{\dot \fx\in \bE} \left\{\<\hat\fx,\dot \fx\> -\frac{1}{2} \cF(\fx,\dot \fx)^2\right\}
\end{equation}
for any point $\fx = (x,\theta)\in \bM$ and any co-tangent vector $\hat\fx = (\hat{x},\hat\theta) \in \bR^2 \times \bR$. Let us emphasize that the curvature penalty $\cC$, the metric $\cF$, and the Hamiltonian $\cH$ have simple and explicit expressions for the models of interest,  presented \cref{sec_ConvexityTheorem}.

Once the geodesic distance map $\cU_\fp$ is known, 
a geodesic path $\cG$ from the source point $\fp$ to an arbitrary target point $\fx \in \bM$ can be backtracked by solving a gradient descent-like ODE backwards in time, see ~\cite{mirebeau2019hamiltonian} for a discussion of suitable numerical methods. Specifically, denoting by $T=\cU_{\fp}(\fx)$ the arrival time, one sets $\cG(T) = \fx$, and 
\begin{equation}
\label{eq_GeoODE}
\cG^\prime(\varrho) = \fV(\cG(\varrho)),\ \forall \varrho \in ]0,T],
\end{equation}
where the vector field $\fV$, referred to as geodesic flow, is defined from the minimal action map $\cU_\fp$ as follows
\begin{equation}
\label{eq_SchemeV}
\fV(\fx) = d\cH_\fx(d\,\cU_{\fp}(\fx)).
\end{equation}

In the remaining of this paper, we remove the dependency on the points $\fp$ of the minimal action map $\cU_\fp$.  

\subsection{Discretization of the eikonal equation} 
\label{subsec_HFM}

We describe in this section the construction of a finite difference scheme approximating the generalized eikonal equation~\eqref{eq_HamiltonEikonal}, and the geodesic flow \eqref{eq_SchemeV}. 
These numerical methods eventually allow to compute paths globally minimizing the curvature-penalized energy \eqref{eq_SecOrderLength}, as in~\cref{subsec_CurvatureGeodesic}.
Our approach follows the framework of the Hamiltonian fast marching (HFM) method~\cite{mirebeau2018fast,mirebeau2019riemannian,mirebeau2019hamiltonian}.

The HFM method takes its name from a specific representation or approximation of the Hamiltonian~\eqref{eq_Hamiltonian} involved in the eikonal equation~\eqref{eq_HamiltonEikonal}. For the models considered in this paper, this reads 
\begin{equation}
\label{eq_SchemeApprox}
	2 \cH_\fx(\hat \fx) \approx \max_{1 \leq k \leq K} \sum_{1 \leq i \leq I} \rho_{ik}(\fx) \< \hat \fx,\dfe_{ik}\>_+^2,
\end{equation}
for any point $\fx\in \bM$ and co-tangent vector $\hat \fx$. The choice of the integers $I,K$, of the \emph{non-negative} weights $\rho_{ik}(\fx)\geq 0$, and of the offsets with \emph{integer coordinates} $\dfe_{ik} \in \bZ^3$, constitutes the main originality of the HFM method~\cite{mirebeau2018fast,mirebeau2019hamiltonian} and is discussed in detail below and in \cref{sec_ConvexityTheorem}. The offsets often depend on the base point, $\dfe_{ik} = \dfe_{ik}(\fx)$, like the weights $\rho_{ik} = \rho_{ik}(\fx)$, but this is omitted in Eq.~\eqref{eq_SchemeApprox} and similar formulas for readability.

\subsubsection{The finite difference scheme}
The HFM method expects the domain $\bM=\Omega\times\bS^1$ to be discretized on a Cartesian grid 
\begin{equation}
\label{eqdef_grid}
	\bM_h:=(\Omega \cap h\bZ^2)\times(h\bZ/ 2\pi\bZ),
\end{equation} 
where $h=2\pi/N_{\theta}$ is the grid scale with $N_{\theta}$ being the number of discrete orientations: $h\bZ/ 2\pi\bZ = \{0,h,2h,\cdots,(N_\theta-1)h\}$.

Let $\cU : \bM \to \bR$ be a smooth function, let $\fx\in \bM_h$ and let $\dfe \in \bZ^3$. Then one has the first-order approximation 
\begin{equation}
	\label{eq_SchemeAsy}
\<d\cU(\fx),\dfe\>_+^2 =  \left(\frac{\cU(\fx) - \cU(\fx- h\dfe)}{h}\right)_+^2 + \cO(h),
\end{equation}
which only involves values of $\cU$ at the grid points $\fx$ and $\fx-h\dfe$, and is thus suitable for defining a finite difference scheme on $\bM_h$. If $\fx-h\dfe$ falls outside $\bM_h$, or if the segment $[\fx,\fx-h\dfe]$ intersects an obstacle introduced in the domain as in~\cref{subsec_ExtendCG} below, then we set $\cU(\fx- h\dfe)=\infty$ in Eq.~\eqref{eq_SchemeAsy}. This convention implements outflow boundary conditions.

Inserting Eq.~\eqref{eq_SchemeAsy} into Eq.~\eqref{eq_SchemeApprox} we obtain a finite differences approximation of $\cH_\fx(d \cU(\fx))$, with first-order $\cO(h)$ error w.r.t.~the grid scale. This yields the following discretization of the eikonal equation \eqref{eq_HamiltonEikonal}: find $u : \bM_h \to \bR$ such that 
\begin{equation}
\label{eq_SchemeMax}
	\max_{1 \leq k \leq K} \sum_{1 \leq i \leq I} \rho_{ik}(\fx) \big(u(\fx)-u(\fx-h\dfe_{ik})\big)_+^2 = h^2 \psi(\fx)^2
\end{equation}
for all $\fx \in \bM_h \sm \{\fp\}$, and $u(\fp)=0$. The specific form of this numerical scheme allows to solve it very efficiently, see \cref{sec_HFM}. The stencil of the scheme is defined as 
\begin{equation}
\label{eqdef:stencil}
	\cS(\fx) := \{\fy_{ik} := \fx-h \dfe_{ik}; 1 \leq k \leq K, 1 \leq i \leq I\}.
\end{equation}

\subsubsection{Approximation of a directional first-order derivative}
For the purposes of this paper, we recall one result of the HFM framework \cite[Proposition 1.1]{mirebeau2018fast}, which allows in \cref{sec_ConvexityTheorem} to approximate the Hamiltonians of the models considered in this paper in the desired form \eqref{eq_SchemeApprox}. Specifically, given a vector $\dfv\in\bE\cong \bR^3$ and a relaxation parameter $\ve>0$, one has for all co-vectors $\hat{\fx}$ :
\begin{equation}
\label{eq_HamiltonianAsym}
	\<\hat\fx,\dfv\>_+^2 = 
	\sum_{1 \leq j \leq J} 
	\rho^{\ve}_{j}(\dfv)\<\hat \fx, \dfe_{j}\>^2_+ 
	+\|\hat\fx\|^2\cO(\ve^2),	
\end{equation}
where $J=6$, 
the weight $\rho^{\ve}_{j}(\dfv)\geq 0$ is non-negative, and 
the offset 
$\dfe_j = \dfe_j^\ve(\dfv) \in \bZ^3$ has integer components, for all $1 \leq j \leq J$,
consistently with Eq.~\eqref{eq_SchemeApprox}. 
The construction of these weights and offsets is non-trivial and involves a tool from low-dimensional lattice geometry known as Selling's decomposition of positive quadratic forms \cite{mirebeau2018fast}. 
It has the following purpose: when $\hat \fx = d\cU(\fx)$ is the differential of a smooth function, then using \cref{eq_HamiltonianAsym,eq_SchemeAsy} and recalling that the offsets have integer components, one can approximate the directional derivative $\<d\cU(\fx),\dfv\>_+^2$ using finite differences whose structure is compatible with the efficient fast marching numerical method, see \cref{subsec:solving_distance} below. 

An analogous result holds in two dimensions: given $\dot v\in \bR^2$, $\ve>0$, one has for all $\hat x \in \bR^2$
\begin{equation}
\label{eq_HamiltonianAsym2}
	\<\hat x,\dot v\>_+^2 = 
	\sum_{1 \leq j \leq J'} 
	\rho^\ve_j(\dot v)\<\hat x, \dot e_j\>^2_+ 
	+\|\hat x\|^2\cO(\ve^2),	
\end{equation}
with $J'=3$ terms, and where $\rho^\ve_j(\dot v)\geq 0$ and $\dot e_j = \dot e_j^\ve(\dot v)\in \bZ^2$ for all $1 \leq j \leq J'$.\\

\subsubsection{The geodesic flow vector field $\fV$}
By differentiating Eq.~\eqref{eq_SchemeApprox}, we obtain 
\begin{equation}
	d \cH_\fx(\hat \fx) \approx \sum_{1 \leq i \leq I} \rho_{ik_*}(\fx) \<\hat \fx, \dfe_{ik_*}\>_+ \dfe_{ik_*},
\end{equation}
where $1 \leq k_* \leq K$ is the index for which the maximum (\ref{eq_SchemeApprox}, right) is attained, by the envelope theorem. Inserting the finite differences \eqref{eq_SchemeAsy} we can approximate the geodesic flow vector \eqref{eq_SchemeV} 
\begin{equation}
\label{eq_SchemeMaxV}
	\fV(\fx) \approx h^{-1} \sum_{1 \leq i \leq I} \rho_{ik_*}(\fx) \big(u(\fx) - u(\fx-\dfe_{i k_*})\big)_+ \dfe_{ik_*},
\end{equation}
where $u : \bM_h \to \bR$ denotes the numerical scheme solution \eqref{eq_SchemeMax}.

\section{Convexity-constrained Geodesic Models}
\label{sec_ConvexityTheorem}
We introduce three new geodesic models which feature both a convexity shape prior and a penalization of curvature. 
More precisely, these models impose a constraint on the \emph{sign} of the curvature, as motivated by the following definition.
\begin{definition}[\textbf{Simple Closed Convex Curves}]
\label{def_Convex} 
A simple closed planar curve $\gamma$, smooth and parametrized in a counter-clockwise order, is said convex iff its curvature $\kappa$ in Eq.~\eqref{eq_Curvature} is non-negative. 
\end{definition}	

In the remaining of this section, we thus design curvature penalty functions $\kC$, which partly coincide with the penalty $\cC$ of the RSF, Dubins and EM elastica models, but also enforce the non-negativity of the curvature. 
We then derive the corresponding geodesic metrics $\kF$, Hamiltonians $\kH$ and their discretizations in the HFM framework \eqref{eq_SchemeApprox}. 
In the text, the convexity-constrained objects $\kC$, $\kF$, and $\kH$ are distinguished from their classical counterparts $\cC$, $\cF$, and $\cH$ by the choice of font. 
Further discussion of the search space for geodesic paths, ensuring that the physical projection is simple and closed, 
is postponed to Section~\ref{sec_SearchingSpace}.

A curvature penalty $\kC : \bR \to ]0,\infty]$, enforcing convexity by imposing a non-negative curvature, should obey by design
\begin{equation}
\label{eq_ConvexCurvature}
	\kC(\beta \kappa) = \infty, \text{~if~} \beta\kappa<0.
\end{equation}
The corresponding geodesic metric $\kF : \bM \times \bE \to [0,\infty]$, defined by Eq.~\eqref{eq_GeneralMetric}, thus satisfies 
\begin{equation}
\label{eq_GenericCharact}
\kF(\fx,\dot\fx)=\infty, ~\text{if~} \dot\theta<0 \text{~or~} \dx\neq \dfnt\|\dx\|.
\end{equation} 
for any point  $\fx=(x,\theta)\in\bM$ and any vector $\dot\fx=(\dx,\dot\theta)\in\bE$. 
The associated energy $\kL$ of a smooth curve $\Gamma : [0,1] \to \bM$ is 
\begin{equation}
\label{eq_LengthConvexity}	
\kL(\Gamma):=\int_0^1\psi(\Gamma(\varrho))\,\kF(\Gamma(\varrho),\dot\Gamma(\varrho))d\varrho.
\end{equation}
In view of the properties of the metric \eqref{eq_GenericCharact}, any curve $\Gamma = (\gamma, \eta)$ with finite energy must obey the orientation-lifting relation \eqref{eq_TurningAngle}, and have a non-decreasing angular component $\dot \eta \geq 0$. Finally, we infer from Eqs.~\eqref{eq_Hamiltonian} and \eqref{eq_GenericCharact} that the Hamiltonian $\kH$ of a convexity-constrained model obeys
\begin{equation}
\label{eq_ConvexH_props}
	\kH((\pm \dfnt^\perp,0)) = \kH((-\dfnt,0)) = \kH((\mathbf 0,-1)) = 0,
\end{equation} 
where $\dfnt^\perp$ denotes the rotation of $\dfnt$ by $\pi/2$.

\noindent\emph{Notation:} In the rest of this section, we fix a base point $\fx = (x,\theta) \in \bM$, a vector $\dot \fx = (\dot x,\dot \theta) \in \bE$, and a co-vector $\hat \fx = (\hat x,\hat \theta)$.

\subsection{Convexity-constrained RSF Model}
\label{subsec_ConvexRSF}
In the \emph{classical} RSF model, the curvature penalty function is defined as 
\begin{equation}
\label{eq_CRS}
	\cC^\rRS(\beta\kappa)=\sqrt{1+(\beta\kappa)^2}.
\end{equation}
The RSF geodesic metric $\cF^\rRS$ is obtained by incorporating $\cC^\rRS$ in the general expression~\eqref{eq_GeneralMetric}, and thus reads~\cite{duits2018optimal}
\begin{equation}
\label{eq_RSFMetric}
\cF^{\rRS}(\fx,\dot\fx)=
\sqrt{\|\dx\|^2+(\beta\dot\theta)^2} \quad \text{if~}\dx=\dfnt\|\dx\|,
\end{equation}
and $\cF(\fx,\dot \fx)=\infty$ otherwise. 
The Hamiltonian $\cH^\rRS$, obtained by specializing \eqref{eq_Hamiltonian} to the RSF metric $\cF^\rRS$, admits the following closed form expression~\cite{mirebeau2018fast}
\begin{equation}
\label{eq_OriginalHamiltonian}
2 \cH^\rRS_\fx(\hat\fx)= \< \hat{x},\dfnt\>_+^2+(\hat\theta/\beta)^2.
\end{equation}

In this section, we consider a \emph{convexity-constrained} RSF model defined, in view of~\eqref{eq_ConvexCurvature}, via the curvature penalty function 
\begin{equation}
	\kC^\rRS(\beta \kappa) = \sqrt{1+(\beta \kappa)^2} \quad \text{if } \beta \kappa \geq 0,
\end{equation}
and $\kC(\beta \kappa) = \infty$ otherwise. 
The corresponding convexity-constrained RSF metric reads, by Eq.~\eqref{eq_GeneralMetric}
\begin{equation*}
\kF^\rRS(\fx,\dot\fx)=\sqrt{\|\dx\|^2+(\beta\dot\theta)^2}
\quad \text{~if~} \dot\theta\geq 0\text{~and~}\dx=\dfnt\|\dx\|,
\end{equation*}
and $\kF^\rRS(\fx,\dot\fx)=\infty$ otherwise.

\begin{proposition}
The Hamiltonian of the convexity-constrained RSF metric $\kF^\rRS$ reads 
\begin{equation}
\label{eq_RSFConvexH}
2\kH^\rRS_\fx(\hat\fx)=\< \hat{x},\dfnt\>_+^2+(\hat\theta/\beta)_+^2,
\end{equation}
for any point $\fx=(x,\theta)\in\bM$ and 
co-vector $\hat\fx=(\hat{x},\hat{\theta})\in\bR^3$.
\end{proposition}
\begin{proof}
Observing that $\dot\theta^2=\dot\theta^2_++\dot\theta^2_-$, we reformulate the convexity-constrained RSF metric $\kF^\rRS$ as follows	
\begin{align*}
&\kF^\rRS(\fx,\dot\fx)^2=\nonumber\\
&\< \dx,\dfnt\>_+^2+\infty\< \dx,\dfnt\>_-^2+\beta^2\dot\theta_+^2
+\infty\dot\theta_-^2+
\infty\<\dx,\dfnt^\perp\>^2,
\end{align*}
with the convention $0\times\infty=0$. 
Noting that $[(\dfnt,0),\allowbreak (\dfnt^\perp,0),\allowbreak (\mathbf 0,1)]$ is an orthonormal basis of $\bE\cong \bR^2 \times \bR$, and by general properties of Legendre-Fenchel duality \eqref{eq_Hamiltonian} for quadratic functions, we obtain the expression~\eqref{eq_RSFConvexH} for the Hamiltonian~$\kH^\rRS$. 
\end{proof}

We obtain using Eq.~\eqref{eq_HamiltonianAsym2} an approximate decomposition of the convexity-constrained RSF Hamiltonian $\kH^{\rRS}$, with a form that fits the HFM framework, and an $\cO(\ve^2)$ error,
\begin{equation*}	
	2\kH^{\rRS}_\fx(\hat\fx) \approx 
	\sum_{1\leq j\leq J^\prime}\rho^\ve_j(\dfnt)\<\hat{x},\dot e_j\>_+^2
	+(\hat\theta/\beta)_+^2.
\end{equation*}
Recall that $J'=3$, that $\rho^\ve_j(\dfnt) \geq 0$ is a non-negative weight, and that $\dot e_j= \dot e_j^\ve(\dfnt) \in \bZ^2$ is a two dimensional offset with integer components, for all $1 \leq j \leq J'$. 
For comparison, the original RSF Hamiltonian~\eqref{eq_OriginalHamiltonian} admits a similar decomposition, used in \cite{mirebeau2018fast}, 
except for the last term which reads $(\hat\theta/\beta)^2$.

Using finite differences \eqref{eq_SchemeAsy} we discretize the operator $\kH^\rRS_\fx(d \cU(\fx))$ of the eikonal equation \eqref{eq_HamiltonEikonal}, with consistency error $\cO(h+\ve^2)$. The scheme involves a stencil $\cS:=\cS^{\rRS}$ consisting of $I = J'+1 = 4$ neighbor points of $\fx$ on the grid $\bM_h$. Namely
\begin{equation}
\cS^{\rRS}(\fx):=\{\fy_i=\fx-h\dfe_i;\,1\leq i\leq I\},	
\end{equation}
 where $\dfe_i\in\bZ^3$ is defined as $\dfe_i=(\dot e_i,0)$ for $1\leq i\leq I-1$, and $\dfe_i=(\mathbf 0,1)$ for $i=I$. 

\subsection{Convex-constrained Dubins Car Model}
\label{subsec_ConvexDubins}

In the \emph{classical} Dubins car model, the curvature penalty function $\cC:=\cC^\rmD$ is obtained by thresholding the curvature $\kappa$, so that 
\begin{equation}
\label{eq_CDubins}
	\cC^\rmD(\beta\kappa) = 1 \quad \text{if } |\beta\kappa|\leq 1, 
\end{equation}
and $\cC^\rmD(\beta\kappa)=\infty$ otherwise. 
The Dubins geodesic metric $\cF^\rmD$ is obtained by incorporating the penalty function $\cC^\rmD$ in the general expression~\eqref{eq_GeneralMetric}, and thus reads
\begin{equation}
\cF^{\rmD}(\fx,\dot\fx)=\|\dx\|, \quad
\text{if~}\dx=\|\dx\|\dfnt\text{~and~}|\beta\dot\theta|\leq\|\dx\|,
\end{equation}
and $\cF^{\rmD}(\fx,\dot\fx)=\infty$ otherwise.
The Hamiltonian $\cH^\rmD$ of the Dubins metric $\cF^\rmD$, defined by Eq.~\eqref{eq_Hamiltonian}, admits the following closed form expression~\cite{mirebeau2018fast}
\begin{equation}
\label{eq_DubinsH}
2\cH^\rmD_\fx(\hat\fx)	=\max\left\{\<\hat\fx, \dfq^\rmD_+\>_+^2,\ \<\hat\fx, \dfq^\rmD_-\>_+^2\right\},
\end{equation}
where $\dfq^\rmD_+:=(\dfnt,1/\beta)$ and $\dfq^\rmD_-:=(\dfnt,-1/\beta)$. 
The two vectors $\dfq^\rmD_+,\dfq^\rmD_-\in \bE$ should be regarded as extremal controls, corresponding to a vehicle moving in circles of radius $\beta$,  respectively in a counter-clockwise and a clockwise manner.

In this section, we introduce a \emph{convexity-constrained} Dubins model defined, in view of \eqref{eq_ConvexCurvature}, by the curvature penalty function 
\begin{equation}
	\kC^\rmD(\beta \kappa) = 1 \quad \text{if } 0 \leq \beta \kappa \leq 1,
\end{equation}
and $\kC^\rmD(\beta \kappa) = \infty$ otherwise. 
The corresponding convexity-constrained Dubins metric reads, by \eqref{eq_GeneralMetric}
\begin{equation}
\kF^\rmD(\fx,\dot\fx)=\|\dx\|,~\text{if~}\dx=\|\dx\|\dfnt\text{~and~}0\leq \beta\dot\theta\leq\|\dx\|,
\end{equation}
and $\kF^\rmD(\fx,\dot\fx)=\infty$, otherwise.

\begin{proposition}
The convexity-constrained Dubins Hamiltonian, denoted by $\kH^\rmD$, reads 
\begin{equation}
\label{eq_ConvexDubinsH}
2\kH^{\rmD}_\fx(\hat\fx)=\max\Big\{\<\hat\fx, \dfq^{\rmD}_+\>_+^2,\<\hat{x}, \dfnt \>_+^2\Big\}.
\end{equation}
\begin{proof}
By definition~\eqref{eq_Hamiltonian}, and positive homogeneity of $\kF^\rmD$, one has
\begin{align}
\label{eq_DubinsHamiltonianOpt}
\kH^\rmD_\fx(\hat\fx)&= \sup_{\dot \fx\in \bE} \left\{\<\hat\fx,\dot \fx\> -{\ts\frac{1}{2}} \kF^\rmD(\fx,\dot \fx)^2\right\}\nonumber\\
&=\sup_{\dot \fx\in\bE} \left\{ {\ts\frac{1}{2}}\<\hat\fx,\dot \fx\>^2_+;\, \kF^\rmD(\fx,\dot\fx)\leq 1\right\}.
\end{align}
The convex optimization problem \eqref{eq_DubinsHamiltonianOpt} is posed on the set 
\begin{align*}
\cB^\rmD(\fx)&=\{\dot\fx\in\bE;\, \kF^\rmD(\fx,\dot\fx)\leq 1\}\\
&=\{(a\dfnt,b/\beta)\in\bE;\, 0\leq b\leq a\leq 1\}.
\end{align*}
The set $\cB^{\rmD}(\fx)$, referred to as the control set of the  metric $\kF^{\rmD}$, is a right-angled triangle, with vertices $(\mathbf{0},0)$, $\dfq^{\rmD}_+,(\dfnt,0) \in \bE \approx \bR^3$. 
Since the mapping $\dot \fx \mapsto \<\hat\fx,\dot \fx\>^2_+$ is convex, the maximal value is attained at one of the vertices of $\cB^{\rmD}(\fx)$, and therefore
\begin{equation*}
2\kH^{\rmD}_\fx(\hat\fx)=\max\{0,\<\hat\fx, \dfq^{\rmD}_+\>^2_+,\<\hat{\fx},(\dfnt,0) \>^2_+\}.
\end{equation*}
The announced result follows. Note that the Hamiltonian of the standard Dubins model \cref{eq_DubinsH} is computed using a similar argument in~\cite{mirebeau2018fast}.
\end{proof}
\end{proposition}

The two vectors $\dfq^\rmD_+ = (\dfnt, 1/\beta)$ and $(\dfnt, 0)$, involved implicitly in Eq.~\eqref{eq_ConvexDubinsH}, should be regarded as extremal controls corresponding respectively to a vehicle moving in circles of radius $\beta$ counter-clockwise, or moving in a straight line.

We obtain using Eqs.~\eqref{eq_HamiltonianAsym} and~\eqref{eq_HamiltonianAsym2} an approximate decomposition of the convexity-constrained RSF Hamiltonian $\kH^\rmD$, with a form that fits the HFM framework, and an $\cO(\ve^2)$ error,
\begin{align}
	2\kH^{\rmD}_\fx(\hat\fx)
	\approx\max\Bigg\{
	\sum_{1\leq j \leq J}\rho_j^\ve(\dfq_+^{\rmD})\<\hat\fx,\dfe_j\>_+^2,&\nonumber\\
	\label{eq_ConvexDubinsDecomp}
	\sum_{1\leq j\leq J^\prime}\rho_j^\ve(\dfnt)\<\hat{x}, \dot e_j\>_+^2&\Bigg\}.
\end{align}
Recall that $J=6$ and $J'=3$, that $\rho_j^\ve(\dfq_+^{\rmD})$ and $\rho_j^\ve(\dfnt)$ are non-negative weights, and that $\dfe_j = \dfe_j^\ve(\dfq_+^D)\in \bZ^3$ and $\dot e_j = \dot e_j^\ve(\dfnt)\in \bZ^2$ are offsets with integer coordinates.

Using the finite differences \eqref{eq_SchemeAsy}, we discretize the operator $\kH^\rmD_\fx(d \cU(\fx))$ of the eikonal equation \eqref{eq_HamiltonEikonal}, with consistency error $\cO(h+\ve^2)$. The scheme involves a stencil $\cS:=\cS^\rmD$ consisting of $I = J+J' = 9$ neighbor points of $\fx$ on the grid $\bM_h$. Namely
\begin{equation}
\cS^\rmD(\fx):=\{\fy_i=\fx-h\dfe_i\in\bM_h;\, 1\leq i \leq I\},	
\end{equation}
where the offsets are defined by $\dfe_i=\dfe_i^\ve(\dfq_+^{\rmD})$ for $1\leq i\leq J$, and $\dfe_i=(\dot{e}_{i-J}^\ve(\dfnt),0)$ for $J+1 \leq i \leq I$. 
For comparison, the Hamiltonian associated to the original Dubins model  $\cH^\rmD$ \eqref{eq_DubinsH} admits an approximate decomposition similar to \eqref{eq_ConvexDubinsDecomp} but involving the vector $\dfq^\rmD_-$ instead of $\dfnt$, and a finite differences scheme based on a $J+J=12$ points stencil~\cite{mirebeau2018fast}.

\subsection{Convexity-constrained EM Elastica Model}
\label{subsec_ConvexGeo}
In the \emph{classical} EM elastica model, the curvature penalty function is defined as 
\begin{equation}
	\cC^\rEM(\beta\kappa)=1+(\beta\kappa)^2.
\end{equation}
This quadratic curvature penalty is intermediate between the quasi-linear penalty of the RSF model \eqref{eq_CRS}, and the hard threshold of the Dubins model \eqref{eq_CDubins}.  
As a result, the EM elastica model assigns a high cost to path sections with large curvature, in contrast to the Dubins model which forbids them, and to the RSF model which tolerates infinite curvature (angular paths), see \cite{mirebeau2018fast}. 
The elastica geodesic metric $\cF^\rEM$ reads, by \eqref{eq_GeneralMetric} 
\begin{equation}
\label{eq_FEMMetric}
\cF^\rEM(\fx,\dot\fx)=
\|\dx\|+\frac{(\beta\dot\theta)^2}{\|\dx\|}
\quad \text{if~} \dx=\dfnt\|\dx\|,
\end{equation}
with $\cF^\rEM(\fx,\mathbf 0) = 0$ 
and $\cF^\rEM(\fx,\dot\fx)=\infty$ otherwise.
The elastica Hamiltonian $\cH^\rEM$ is derived from the metric $\cF^\rEM$ by Legendre-Fenchel duality \eqref{eq_Hamiltonian}, and in \cite{mirebeau2018fast} it is shown equal to 
\begin{align}
\label{eq_HEM_Alg}
 2\cH_\fx^\rEM(\hat\fx)&=\frac{1}{4}\left(\<\hat{x}, \dfnt\> + \sqrt{\<\hat{x}, \dfnt\>^2 + (\hat\theta/\beta)^2}\,\right)^2\\
\label{eq_HEM_Int}
&= \int_{-\pi/2}^{\pi/2}  \big\<\dfq(\theta,\vp), \hat \fx\big\>_+^2 \cos \vp\, d\vp,
\end{align}
where we denoted
\begin{equation}
\label{eq_qEM}
	\dfq(\theta,\vp) := \textstyle{\frac{\sqrt 3}2}(\dfnt \cos \vp, \beta^{-1} \sin \vp) \in \bE \cong \bR^2 \times \bR.
\end{equation}
Somewhat curiously, the integral form Eq.~\eqref{eq_HEM_Int} is more suitable for the HFM solver framework than the algebraic form \eqref{eq_HEM_Alg} of the Hamiltonian. 
Indeed, using the Fejer quadrature rule\footnote{The integral \eqref{eq_HEM_Int} has a cosine weight over the interval $[-\pi/2,\pi/2]$, which is equivalent to a sine weight over $[0,\pi]$, thus suitable for the Fejer rule.} with $L$ points we obtain with an $\cO(1/L^2)$ error
\begin{equation*}
	2\cH^\rEM_\fx(\hat\fx)\approx
	\sum_{1 \leq l \leq L} w_l \<\hat\fx,\dfq(\theta,\vp_l)\>_+^2, 
\end{equation*}
with suitable weights $w_l\geq 0$, and with the quadrature nodes 
$\vp_l := (2l-L-1)\pi/(2L)$, for all $1 \leq l \leq L$.
From this point, the approximate decomposition~\eqref{eq_HamiltonianAsym} of each of the terms $\<\hat\fx,\dfq(\theta,\vp_l)\>_+^2$, yields with Eq.~\eqref{eq_SchemeAsy} a finite differences scheme with an $\cO(h+\ve^2+L^{-2})$ error and a stencil of $L J$ points~\cite{mirebeau2018fast}. In practice, we use $L=5$ which yields enough accuracy, and use the relaxation parameter $\ve=0.1$, whereas $J=6$ by definition.

In contrast with the RSF and Dubins models, we design the \emph{convexity-constrained} EM elastica model through a modification of Hamiltonian $\cH^\rEM$ \eqref{eq_HEM_Int}, rather than of the curvature penalty function $\cC^\rEM$.
More precisely, we let (notice the integral bounds)
\begin{equation}
\label{eq_ConvexHEM_Int}
2\kH_\fx^{\rEM}(\hat\fx) := \int_0^{\pi/2} \big\<\dfq(\theta,\vp), \hat \fx\big\>_+^2 \cos\vp\, d\vp.
\end{equation}	
The vectors $\dfq(\theta,\vp)$ in Eqs.~\eqref{eq_HEM_Int} and \eqref{eq_ConvexHEM_Int} should be regarded as controls. 
When $\vp \geq 0$ as in Eq.~\eqref{eq_ConvexHEM_Int}, the third component of $\dfq(\theta,\vp)$ is non-negative, see Eq.~\eqref{eq_qEM}, and this control thus corresponds to a vehicle rotating clockwise. Note also that $\kH^\rEM$ obeys Eq.~\eqref{eq_ConvexH_props}. In the following, we compute the metric $\kF^\rEM$ and curvature cost function $\kC^\rEM$ corresponding to $\kH^\rEM$, and show that they obey Eqs.~\eqref{eq_GenericCharact} and \eqref{eq_ConvexCurvature} as desired.

Before that, let us conclude the description of the numerical scheme for the convexity-constrained EM elastica model. We have by~\eqref{eq_HamiltonianAsym} the approximate decomposition, with $\cO(\ve^2+L^{-1})$ error
\begin{equation*}
	2\kH_\fx^{\rEM}(\hat\fx) \approx \sum_{1 \leq l \leq L} \tilde w_l \sum_{1 \leq j \leq J} \rho_{jl} \< \hat \fx, \dfe_{jl}\>_+^2,
\end{equation*}
where $\rho_{jl} = \rho_j^\ve(\dfq(\theta,\vp_l))$ and $\dfe_{jl} = \dfe_j^\ve(\dfq(\theta,\vp_l))$, and where
\begin{equation}
\label{eqdef:tw}
	\tilde w_l := 
	\begin{cases}
	w_l &\text{ if } \vp_l > 0,\\
	w_l/2 &\text{ if } \vp_l = 0,\\
	0 &\text{ else},
	\end{cases}
\end{equation}
is an adaptation of the Fejer rule for the integral \eqref{eq_ConvexHEM_Int} over the half domain $[0,\pi/2]$. Alternatively, another consistent approximation of Eq.~\eqref{eq_ConvexHEM_Int}, albeit with a larger consistency error, can be achieved by retaining the original Fejer weights $w_l$, but introducing $\tilde \rho_{jl} := \rho_{jl}$ if the third component of $\dfe_{jl}$ is non-negative, and $\tilde \rho_{jl} := 0$ otherwise%
\footnote{Indeed, note that $\dfe_{jl}$ is almost aligned with $\dfq(\theta,\vp_l)$, due to \eqref{eq_HamiltonianAsym}, at least when $\rho_{jl}$ is sufficiently positive. Hence the conditions $\vp_l \geq 0$ and $\<\dfe_{jl},(\mathbf 0,1)\> \geq 0$ are closely related.}.

A consistent approximation of the eikonal equation operator $\kH(d \cU(\fx))$ is obtained by introducing the finite difference approximations \eqref{eq_SchemeAsy}, resulting in the discretized system of equations
\begin{equation}
\label{eq_SchemeEM}
\sum_{1\leq l\leq L} \tilde w_l\sum_{1 \leq j \leq J}\rho_{jl}\left(\frac{u(\fx)-u(\fx- h\dfe_{jl})} h\right)_+^2  = \psi(\fx)^2,
\end{equation}
which is a special case of \eqref{eq_SchemeMax} with $K=1$.
The numerical scheme of the original EM elastica uses LJ points, but in view of Eq.~\eqref{eqdef:tw} the stencil $\cS^\rEM$ of the convexity-constrained variant only contains $I = \lceil L/2 \rceil J$ points at most. 

In the rest of this section, we obtain closed form expressions of the Hamiltonian, metric, and curvature penalty, of the convexity-constrained elastica model. This is only motivated by a better understanding of the model, since for all practical purposes the front propagation and geodesic backtracking are implemented using \eqref{eq_SchemeMax} and \eqref{eq_SchemeMaxV} which only require the weights and offsets of the scheme. As a starter, we present a closed form expression of $\kH^\rEM$ in polar coordinates.

\begin{figure}[t]
\centering
\includegraphics[height=3cm]{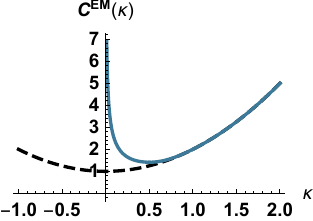}
\hspace{0.3cm}
\includegraphics[height=3cm]{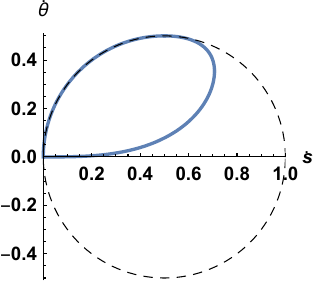}
\caption{
\textbf{Left}: Curvature penalty $\cC^\rEM$ of the EM elastica model (dashed line), and $\kC^\rEM$ of the convexity-constrained variant (solid line). 
\textbf{Right}: Unit vectors in tangent space for the EM elastica model with and without convexity prior. Set of all $(\dot s, \dot \theta)$ such that $\kF^\rEM(\fx,(\dot{s}\dfnt,\dot \theta)) = 1$ (solid line), and likewise for $\cF^\rEM$ (dashed line).}
\label{fig_ControlSet}
\end{figure}

\begin{proposition}
\label{prop_ConvexElastica}
The function $\lambda:[-\pi,\pi]\to\bR$ defined by 
\begin{equation}
\lambda(\phi):=\frac{3}{8} \int_0^{\pi/2} (\cos (\vp-\phi))_+^2 \cos \vp \, d \vp.
\end{equation}
admits the explicit expression  
\begin{equation*}
8\lambda(\phi) = 
\begin{cases}
		0 & \text{if~} \phi \in [-\pi,-\pi/2],\\
		2\cos \phi + 2 \cos \phi \sin \phi & \text{if~} \phi \in [-\pi/2,0],\\
		1+ \cos^2\phi + 2 \cos \phi \sin \phi & \text{if~} \phi \in [0,\pi/2],\\
		1+ \cos^2 \phi +2 \cos \phi & \text{if~} \phi \in [\pi/2,\pi].
\end{cases}
\end{equation*}	
Also, the convexity-constrained elastica Hamiltonian $\kH^\rEM$ reads
\begin{equation*}
\kH^\rEM_\fx(\hat\fx) = r^2 \lambda(\phi), \quad 
\text{when~} \big(\<\hat{x},\dfnt\>,\hat\theta/\beta\big) = r(\cos \phi,\sin \phi),
\end{equation*}
for some $r>0$ and $\phi \in [-\pi,\pi]$.
\end{proposition}

\begin{proof}
	By considering the sign of $\cos(\phi-\vp)$, we find that $\lambda(\phi)$ is the integral of the trigonometric expression $\cos(\vp-\phi)^2 \cos \vp$ over the interval $\emptyset$, $[0,\pi/2+\phi]$, $[0,\pi/2]$, and $[\phi-\pi/2,\pi/2]$ respectively in each of the four distinguished cases. The expression of $\lambda(\phi)$ easily follows.
	Finally we observe that 
\begin{align*}
	\frac 2 {\sqrt 3}\<\dfq(\theta,\vp), \hat \fx\> 
	&= \<\hat x,\dfnt\> \cos \vp + \frac{\hat \theta}\beta \sin \vp\\
	&= r (\cos \vp \cos \phi+ \sin \vp \sin \phi) =  r\cos (\vp-\phi)
\end{align*}
and the expression of $\kH^\rEM$ follows from \eqref{eq_ConvexHEM_Int}.
\end{proof}

The metric $\kF^\rEM$ can be expressed in terms of the Hamiltonian $\kH^\rEM$ using Legendre-Fenchel duality, inverting the relation \eqref{eq_Hamiltonian}: 
\begin{equation}
\label{eq:MetricFromHamiltonian}
\frac 1 2 \kF^\rEM (\fx,\dot \fx)^2 = \sup_{\hat\fx}\Big\{\<\hat\fx,\dot \fx\> - \kH^\rEM_\fx(\hat\fx)\Big\}.
\end{equation} 
Denoting $\dot{s}:= \|\dot x\|$ and $\dot\nu:=\beta\dot\theta$, the metric $\kF^{\rEM}$ reads for any non-zero vector $\dot \fx =(\dot x,\dot \theta) \in \bE$
\begin{align}
\label{eq_FEMC}
&\kF^\rEM(\fx,\dot \fx)^2=\\
&\begin{cases}
+\infty, &\text{if~} \dot\theta < 0 \text{~or~} \dot{x}\neq \dfnt\|\dx\|,\\
\frac{8}{27 \dot\nu} (9\dot{s}\dot\nu^2 + \dot{s}^3 + (\dot{s}^2-3 \dot\nu^2)^\frac{3}{2})
	&\text{if~} 0 < \dot\nu\leq\dot s/2,\\
4(\dot{s}^2 - 2 \dot{s} \dot\nu + 2 \dot\nu^2), 
	&\text{if~} 0 \leq \dot{s}/2 \leq \dot\nu \leq \dot s,\\
(\dot{s} + \dot\nu^2/\dot{s})^2, &\text{if~} 0\leq\dot{s} \leq \dot \nu.
\end{cases}
\nonumber
\end{align}
\begin{proof}[Sketch of proof of \eqref{eq_FEMC}]
The announced expression of $\kF^\rEM$ was obtained with the help of the formal computing program Mathematica$^{\scriptsize\textregistered}$.
We only show here how it can be checked formally, once it is known. Define the Lagrangian
\begin{equation*}
\lag_\fx(\dot \fx) := \frac 1 2 \kF^\rEM(\fx,\dot\fx)^2,
\end{equation*}
as defined from Eq.~\eqref{eq_FEMC}, and denote the Hamiltonian $\kH_\fx(\hat \fx) := \kH_\fx^\rEM(\hat \fx)$. Our objective is to establish \eqref{eq:MetricFromHamiltonian}, in other words that $\lag_\fx$ is the Legendre-Fenchel dual of $\kH_\fx$. This relation is characterized\footnote{And in addition $\lag_\fx (\dot \fx) = \infty$ when $\dot \fx$ is not in the range of $\nabla \kH$, which is readily checked.} by the identity 
	\begin{equation}
	\label{eq:LF_charact}
		\lag_\fx (\nabla \kH_\fx(\hat \fx)) = \kH(\hat \fx)
	\end{equation}
 for any co-vector $\hat \fx\in \bE^*$.
 Recalling that $\kH_\fx(\hat \fx) = r^2 \lambda(\phi)$ from Proposition~\ref{prop_ConvexElastica}, and differentiating in these polar-like coordinates, we obtain
\begin{equation*}		
\nabla \kH_\fx(\hat \fx) = 2 r \lambda(\phi) (\dfnt\cos \phi, \frac{\sin \phi}\beta)  + r \lambda'(\phi) (-\dfnt \sin \phi, \frac{\cos \phi}\beta)
\end{equation*}
In the sequel we assume that $r = 1$, for simplicity and w.l.o.g.\ by homogeneity of $\kH$ and $\lag$. 
In view of \eqref{eq_FEMC}, we define $\dot s$ and $\dot \nu$ by $\nabla \kH_\fx(\hat \fx) = \dot \fx = (\dot x,\dot \theta) = (\dot s\, \dfnt,\dot \nu/\beta)$, in other words
\begin{align*}
	\dot s &= 2 \lambda(\phi) \cos \phi - \lambda'(\phi) \sin \phi,\\
	\dot \nu &= 2 \lambda(\phi) \sin \phi + \lambda'(\phi) \cos \phi.
\end{align*}	
In order to conclude the proof, we need to insert the explicit expression of $\lambda$, distinguishing cases depending on the interval containing $\phi$. 
If $\phi\in [-\pi,\pi/2]$, then $\lambda(\phi) = \lambda'(\phi) = 0$, thus $\dot s = \dot \nu = 0$ and therefore both sides of~\eqref{eq:LF_charact} are zero. 
If $\phi \in [-\pi/2,0]$, then we obtain $8\dot s = 2(2 \cos \phi + 2 \cos \phi \sin \phi) \cos \phi - (-2 \sin \phi - 2 \sin^2 \phi + 2 \cos^2 \phi) \sin \phi$, and likewise $\dot \nu$ is a polynomial function of $\cos \phi$ and $\sin \phi$. A long and tedious sequence of elementary trigonometric identities, which is not presented here, and which is within the grasp of symbolic computation methods, yields the inequality $0 \leq  \dot \nu \leq \dot s/2$ and the identity 
\begin{equation*}
	\frac{8}{27 \dot\nu} (9\dot{s}\dot\nu^2 + \dot{s}^3 + (\dot{s}^2-3 \dot\nu^2)^\frac{3}{2})
	= 2\cos \phi + 2 \cos \phi \sin \phi,
\end{equation*}
which is equivalent to \eqref{eq:LF_charact} when $\phi \in [-\pi/2,0]$. (Use that $\dot s^2-3\dot \nu^2 = (\cos 2 \phi-\sin \phi)^2/16$.) 
Proving \eqref{eq:LF_charact} thus reduces, likewise when $\phi \in [0,\pi/2]$ and $\phi\in [\pi/2,\pi]$, to checking an inequality and an equality between suitable polynomials in the variables $\cos \phi$ and $\sin \phi$, which follows from elementary (yet tedious) computations.
\end{proof}

The curvature penalty of the convexity-constrained EM model can be recovered from the metric and the relation $\kC^\rEM(\beta \kappa) = \kF^\rEM(\fx, (\dfnt, \kappa))$, in view of \eqref{eq_GeneralMetric}. This also amounts to choosing $\dot{s} = 1$ and $\dot \nu = \beta\kappa$ in \eqref{eq_FEMC}, and therefore 
\begin{equation*}
	\kC^\rEM(\dot \nu) 
	= 
	\begin{cases}
	+\infty & \text{if } \nu \leq 0\\
	\sqrt{\frac 8 {27 \dot \nu} (9\dot \nu^2 + 1 + (1-3 \dot \nu^2)^\frac 3 2)} & \text{if } 0 < \dot \nu \leq 1/2\\ 
	2\sqrt{1-2\dot \nu +2 \dot\nu^2} & \text{if } 1/2 \leq \dot \nu \leq 1,\\
	1+\dot \nu^2, & \text{if } 1 \leq \dot \nu.	
	\end{cases}
\end{equation*}
The penalty $\kC^\rEM$ is infinite when the path curvature is negative, as expected, and coincides with the original EM elastica model $\cC^\rEM$ when $1 \leq \dot \nu := \beta \kappa$, in other words when the path curvature exceeds $1/\beta$. Note that straight line segments have finite energy for the convexity-constrained RSF and Dubins models, since $\kC^\rRS(0) = \kC^\rmD(0) = 1$, whereas the convexity-constrained elastica model only allows strictly convex paths, since $\kC^\rEM(0) = \infty$.
The curvature penalty function $\kC^\rEM$, and the set of all $(\dot s, \dot \theta)$ such that $\kF^\rEM(\fx,(\dot s \dfnt,\dot \theta)) = 1$ with $\beta=1$, are illustrated in Fig.~\ref{fig_ControlSet}.

\section{Finding Simple Closed Convex Curves}
\label{sec_SearchingSpace}
We have introduced in \cref{sec_ConvexityTheorem} three convexity-constrained  geodesic metrics $\kF$, on the state space $\bM = \Omega\times \bS^1$ of physical positions and orientations, featuring both a convexity shape prior and a curvature regularization. In this section, we address the problem of tracking orientation-lifted geodesic paths $\cG:=(\mathscr{C},\eta)$, with respect to the chosen metric $\kF$, whose physical projections $\mathscr{C} \in C^2([0,1],\Omega)$ on the physical space $\Omega$ are simple closed and convex. Moreover, we introduce an additional constraint to ensure that the planar curves $\mathscr{C}$ enclose a given point in $\Omega$, so as to accommodate various practical segmentation tasks.
Toward that purpose, we define a reduced domain $\tilde\bM$, by removing an appropriate region which acts as a wall, and choose a specific source and endpoint $\fp$, see~\cref{subsec_ExtendCG}. The work presented in this paper can be regarded as a variant of the circular geodesic model~\cite{appleton2005globally}. It combines constraints on the physical projections  and the angular components of the orientation-lifted curves, so as to control the associated total curvature and to eliminate curves whose physical projections have \emph{self-intersections}.

\subsection{Endpoints and Artificial Obstacles} 
\label{subsec_ExtendCG}
The circular geodesic (CG) model~\cite{appleton2005globally} is a practical method for extracting closed planar geodesic curves within the image domain $\Omega$. Our variant of this model requires the user to fix a planar point $z\in \Omega$, and an orientation lifted point $\fp = (p,\theta_p) \in \bM = \Omega\times \bS^1$, subject to the compatibility condition $\det(p-z,\, \dfn_{\theta_p}) >0$. 

In the context of interactive image segmentation, it is natural to place the point $p$ on the boundary of the target region, with $\theta_p$ denoting the tangent orientation of this boundary in the trigonometric orientation. The point $z$ is placed inside the target convex region, allowing the user to guide the image segmentation in a simple and reliable manner. A crucial ingredient of the CG model is a ray line, i.e.\, a half line, 
introduced within the domain $\Omega$ and defining a cut which locally disconnects its two sides. More precisely, the ray line $\ray$ originating from $z$ and passing through $p$ is regarded as an obstacle. We also introduce a barrier at the position $\theta_p$ in the angular domain $\bS^1$. See \cref{fig_CSG}, where $\ray$ is dotted, the angular coordinate $\theta_p$ of $\fp$ is indicated by the arrow $(\cos\theta_p,\sin\theta_p)$, and an admissible  planar path $\gamma$ is shown.
The accessible domain remaining for the orientation-lifted geodesic paths is thus
\begin{equation}
\label{eqdef_obs}
	\tilde \bM := \{(x, \theta) \in \bM;\,  x \notin \Re_z(p) \text{ and } \theta \neq \theta_p\}.
\end{equation}
Our method exploits a path of minimal energy $\kL$, from the point $\fp$ to itself, and within the obstacle free domain $\tilde \bM$: we solve
\begin{align}
\label{eq_optim}
\inf\Big\{\kL(\Gamma);\, &\Gamma \in C^1([0,1],\bM),\nonumber\\
&\Gamma(0) = \Gamma(1) = \fp,\Gamma(\varrho) \in \tilde \bM, \forall \varrho \in ]0,1[\Big\}.
\end{align}
We show in Section~\ref{subsec_total_curvature} that the physical projection of any candidate minimizer to the problem~\eqref{eq_optim} is a simple closed and convex curve enclosing the point $z$, and we describe Section~\ref{sec_HFM} a global numerical optimization procedure.

\begin{figure}[t]
\centering
\includegraphics[width=8cm]{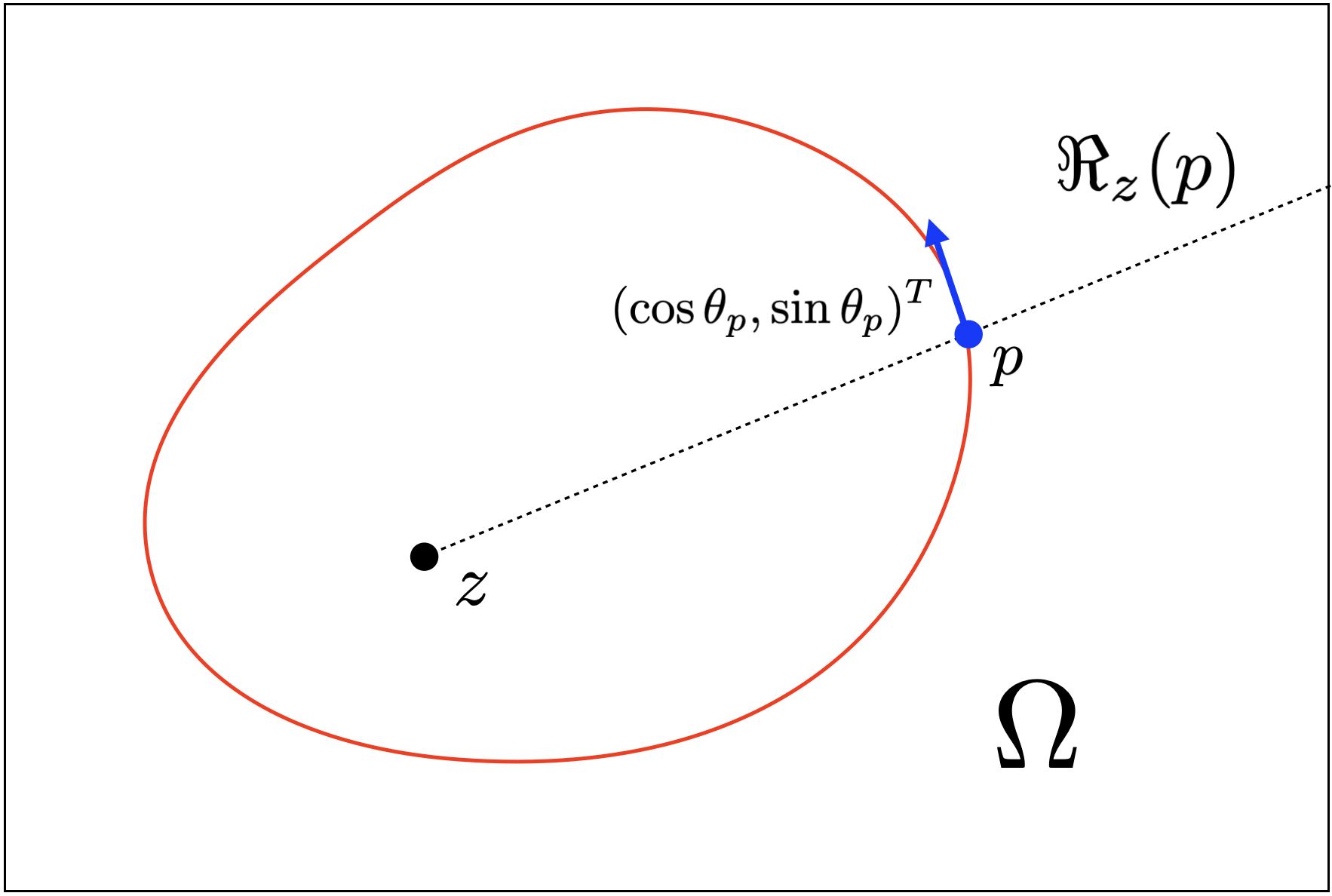}
\caption{The black dash line indicates the ray line $\Re_z(p)$, which originates from the point $z\in\Omega$ (black dot) and passes through $p$ (blue dot). The arrow stands for the direction of $(\cos\theta_p,\sin\theta_p)$, respectively.}
\label{fig_CSG}	
\end{figure}

\subsection{Total Curvature and Curve Simplicity}
\label{subsec_total_curvature}
In this section, we show that the constraints in Eq.~\eqref{eq_optim} ensure the simplicity of the physical projection $\gamma$ of any admissible orientation-lifted curve $\Gamma=(\gamma,\eta)$. 

For that purpose, let us summarize the properties obeyed by a candidate minimizer $\Gamma = (\gamma,\eta)$ to the optimization problem \eqref{eq_optim}:
\begin{align}
\label{cons_bc}
	\gamma(0) &= \gamma(1) = p, &
	\eta(0) &= \eta(1) = \theta_p, \\
\label{cons_obs}
	\gamma(\varrho) &\notin \ray, & \eta(\varrho) &\neq \theta_p, \\
\label{cons_met}
	\dot{\gamma}(\varrho)&=\dfn_{\eta(\varrho)}\|\dot\gamma(\varrho)\|, &
	\dot \eta (\varrho) &\geq 0.
\end{align}
The first line follows from the boundary condition $\Gamma(0) = \Gamma(1) = \fp$. The second line holds for $\varrho \in ]0,1[$ and follows from the choice \eqref{eqdef_obs} of the obstacle-free domain $\tilde\bM$. The third line holds for $\varrho \in [0,1]$, and follows from $\kL(\Gamma)< \infty$ (otherwise $\Gamma$ is not a candidate minimizer) and from the choice of metric, see Eqs.~\eqref{eq_GenericCharact} and \eqref{eq_LengthConvexity}.

The rest of this section will be devoted to the proof of the following proposition. 
\begin{proposition} 
\label{Prop_SCC}
Assume that $(\gamma,\eta) \in C^1([0,1],\bM)$ obeys Eqs.~\eqref{cons_bc}, \eqref{cons_obs}, \eqref{cons_met}, and that $\gamma'$ is non-vanishing over $[0,1]$. Then the curve $\gamma$ is simple, closed, convex, and encloses the point $z$. Furthermore, the total (absolute) 
	curvature of $\gamma$ is $2 \pi$.
\end{proposition}
\begin{proof}
We introduce a  parametric function $\underline \eta : ]0,1[ \to \bR$ such that $\underline \eta(\varrho) \equiv \eta(\varrho) \pmod{2\pi}$, for all $\varrho \in ]0,1[$, and likewise $\underline \theta_p \in \bR$ such that $\underline \theta_p \equiv \theta_p \pmod{2\pi}$.
We may assume that $\underline \eta(0) = \underline \theta_p$, by~(\ref{cons_bc}, right). Observe that $\underline \eta(\varrho)$ is non-decreasing, by~(\ref{cons_met}, right), and that $\underline\eta(\varrho) \notin \theta_p + 2 \pi \bZ$ for all $\varrho \in ]0,1[$, by~(\ref{cons_obs}, right). As a result 
 \begin{equation}
 \label{eq_eta_bound}
 	\underline\theta_p < \underline\eta(\varrho) < \underline\theta_p + 2 \pi,\,
 	\forall \rho \in ]0,1[,
 \end{equation} 
 and finally $\underline \eta(1)=\underline\theta_p + 2 \pi$, by Eq.~(\ref{cons_bc}, right).

The path $\gamma$ has a non-vanishing velocity, and obeys the lifting compatibility condition (\ref{cons_met}, left), hence its curvature is obtained as  $\kappa = \dot \eta/ \|\dot \gamma\|$, see Eq.~\eqref{eq_Curvature}. 
Note that the curvature $\kappa$ is non-negative, by Eq.~(\ref{cons_met}, right), and thus equal to the absolute curvature.
The total curvature of $\gamma$ (absolute or otherwise) is obtained as 
\begin{equation}
\label{eq_TC}
\tcurv(\gamma) =\int_0^1 \kappa(\varrho)\, \|\dot\gamma(\varrho)\|\,d\varrho
= \int_0^1 \dot\eta(\varrho) \,d\varrho,
\end{equation}
and therefore, in view of previously established properties of $\underline \eta$
\begin{equation*}
	\tcurv(\gamma) 
	= \int_0^1 \dot{\underline \eta}(\varrho) \,d\varrho
	= \underline \eta(1) - \underline \eta(0)
	= (\theta_p + 2 \pi)  - \theta_p = 2 \pi.
\end{equation*}
Consider $0\leq \varrho_1<\varrho_2< 1$ such that $\eta(\varrho_1) = \eta(\varrho_2)$, if they exist. Then from Eq.~\eqref{eq_eta_bound} and the monotony of $\eta$ we obtain $\underline \eta(\varrho) = \underline \eta(\varrho_1)$ for all $\varrho_1 \leq \varrho \leq \varrho_2$, hence $\gamma$ that is restricted to $[\varrho_1,\varrho_2]$ is a straight segment, and therefore the tangent lines at $\gamma(\varrho_1)$ and at $\gamma(\varrho_2)$ coincide. It easily follows that there are no three\footnote{If the tangent lines at $\gamma(\varrho_1)$, $\gamma(\varrho_2)$ and $\gamma(\varrho_3)$ are parallel, then $\eta(\varrho_1)\equiv \eta(\varrho_2) \equiv \eta(\varrho_3) \pmod{\pi}$, thus by the pigeonhole principle one has e.g.\ $\eta(\varrho_1)\equiv \eta(\varrho_2) \pmod{2\pi}$, and then by the previous argument the tangents at $\gamma(\varrho_1)$ and $\gamma(\varrho_2)$ are identical.} points of $\gamma$ such that the tangents at these points are pairwise distinct and parallel, and thus that $\gamma$ is the boundary of a convex set $R\subset \bR^2$.

By construction, the boundary $\partial{R}$ intersects the half line $\Re_z(p)$ at the point $p$ and nowhere else. Recalling that $\det(p-z,\, \dfn_{\theta_p}) > 0$ we see that $\Re_z(p)$ is not tangential to $\partial{R}$, and thus $R$ contains the bounded connected component of $\Re_z(p) \setminus \{p\}$, which is the segment $[z,p[$. In other words, the curve $\gamma$ encloses the point $z$, and the announced result follows.
\end{proof}

\noindent\emph{Summary}.
In Sections~\ref{sec_ConvexityTheorem} and~\ref{sec_SearchingSpace}, we have introduced new convexity-constrained geodesic models as variants of the respective classical RSF, Dubins and EM elastica models. In summary, the proposed variant models differ to those original ones in two ways: firstly, the proposed models impose a constraint on the sign of the curvature of curves, leading to new geodesic metrics and associated Hamiltonians which admit the prescribed convexity shape constraint. Secondly, in contrast to the classical curvature-penalized models, the proposed convexity-constrained models track geodesic paths in a subdomain $\tilde\bM$ obtained by subtracting well chosen obstacles to the global domain $\bM$, yielding a bound on the total curvature of the physical projection curves to ensure their simplicity. 

Furthermore, the HFM approach~\cite{mirebeau2018fast} exploits the Hamiltonian expressions to describe those curvature-penalized models, in order to design efficient discretization schemes for numerically solving the associated eikonal equations. This numerical approach differs from~\cite{duits2018optimal} and \cite{chen2017global}, which rely instead on relaxations of the metric, as briefly recalled on Appendix. The HFM approach is more general, customizable, and easily implemented on GPUs.

\section{Numerical Solutions} 
\label{sec_HFM}
Let us consider the path-length quasi-distance function $\cD$ that is defined by
\begin{align*}
	\cD(\fx,\fy) := \inf \Big\{\kL(\Gamma);\, &\Gamma\in C^1([0,1],\tilde \bM), \nonumber\\
	&\ \Gamma(0) = \fx, \Gamma(1) = \fy\Big\},
\end{align*}
for all points $\fx,\fy$ of the obstacle-free domain $\tilde \bM$, see Eq.~\eqref{eqdef_obs}.
We set here the objective of computing the distance $\cD(\fp_0,\fp_1)$ with
\begin{align}
\label{eqdef_endpoints}
	\fp_0 &:= (p+\epsilon \dfn_{\theta_p}, \theta_p+\epsilon), &
	\fp_1 &:= (p-\epsilon \dfn_{\theta_p}, \theta_p-\epsilon),
\end{align}
where $\fp = (p,\theta_p)$ is a user defined orientation-lifted point, see Section~\ref{subsec_ExtendCG}, and $\epsilon>0$ is a small parameter. Letting $\epsilon \to 0$, we recover the original problem \eqref{eq_optim} of finding a path from $\fp$ to itself, within $\tilde \bM$ and minimizing $\kL$. Note that the endpoints $\fp_0,\fp_1$ may be replaced with any small perturbations of $\fp$ whose components lie on the similar side of the obstacle $\Re_z(p)$ in the planar domain, and of the obstacle $\{\theta_p\}$ in the angular domain. 
In practice, we may choose the points of the discretization grid $\bM_h$, see Eq.~\eqref{eqdef_grid}, which are the closest to $\fp$ and obey these geometrical constraints.

We define a geodesic distance map $\cU : \tilde \bM \to [0,\infty]$ by
\begin{equation}
	\cU(\fx) := \cD(\fp_0,\fx).
\end{equation}
It obeys the eikonal equation $\kH(\nabla \cU(\fx) ) = \frac 1 2 \psi(\fx)^2$ on $\tilde \bM \sm \{\fp_0\}$, with $\cU(\fp_0) = 0$ and with outflow boundary conditions on $\partial \tilde \bM$. Here $\kH$ is the Hamiltonian corresponding to one of the convexity-constrained geodesic models, see Section~\ref{sec_ConvexityTheorem}.
A numerical solution $u : \bM_h \to \bR$ to the discretized eikonal equation \eqref{eq_SchemeMax} is computed, using one of the solvers described in the next subsection. Once $u$ is known, a globally optimal path from $\fp_0$ to $\fp_1$ can be extracted by backtracking as defined in Eq.~\eqref{eq_GeoODE}, based on the expression~\eqref{eq_SchemeMaxV} of the geodesic flow.

\subsection{Solving for the Geodesic Distance Map}
\label{subsec:solving_distance}
In this section, we briefly discuss the numerical computation of the unique solution to the discretized eikonal equation \eqref{eq_SchemeMax}, which is quite standard. Indeed, our main contribution lies at the design of the curvature-constrained geodesic metrics, of the scheme coefficients and offsets, and of the artificial obstacles, see~\cref{sec_ConvexityTheorem,sec_SearchingSpace}.
For simplicity, we assume w.l.o.g.\ that the desired path endpoints $\fp_0,\fp_1$ belong to the discretization grid $\bM_h$, see~Eqs.~\eqref{eqdef_grid} and \eqref{eqdef_endpoints}.

For numerical purposes, an array of unknowns $u : \bM_h \to \bR$ is introduced, and initialized to $u(\fp_0) = 0$ and $u = \infty$ elsewhere on $\bM_h \sm \{\fp_0\}$.
In the course of the numerical solver, the unknown is updated by
solving locally the numerical scheme~\eqref{eq_SchemeMax}, at some given point $\fx \in \bM_h\sm \{\fp_0\}$. In such an update, we assign $u(\fx) \gets \ku$, where $\ku$ solves
\begin{equation}
\label{eq_Update}
	\max_{1 \leq k \leq K} \sum_{1 \leq i \leq I} \rho_{ik}(\fx) (\ku - \ku_{ik})_+^2 = h^2 \psi(\fx)^2,
\end{equation}
and where for all $1 \leq i \leq I$ and $1 \leq k \leq K$ one has denoting $\fy_{ik} := \fx-h \dfe_{ik}$
\begin{equation}
\label{eq_neighbor_value}
	\ku_{ik} = 
\begin{cases}
	+ \infty & \text{if } \fy_{ik} \notin \bM_h \text{ or } [\fx, \fy_{ik}] \not\subset \tilde \bM, \\
	+ \infty & \text{if } \fy_{ik} \text{ not } \textsc{Accepted},\\
	u(\fy_{ik}) & \text{else}.
\end{cases}
\end{equation}
The first line serves to apply outflow boundary conditions on $\partial \bM$, and to avoid any front propagation across the obstacles introduced in the  obstacle-free domain $\tilde \bM$,  defined by~\cref{eq_optim}, which ensures that the computation of the geodesic distance values is restrained to $\tilde \bM$. The second line of Eq.~\eqref{eq_neighbor_value} is specific to the fast marching numerical method. Note that letting $\ku_{ik} = \infty$ is equivalent to ignoring the contribution of the stencil point $\fy_{ik}$, since the update equation involves the difference $(\ku - \ku_{ik})_+$. 
Once the values $(\ku_{ik})_{1 \leq i \leq I}^{1\leq k \leq K}$ are gathered and sorted, one can find the update value $\ku$ by solving at most $KI$ univariate polynomial equations of degree two, since Eq.~(\ref{eq_Update}, l.h.s.) is a piecewise quadratic function of $\ku$. A similar procedure was already exploited in \cite{sethian1999fast}.

The ordering of the updates \eqref{eq_Update} depends on the numerical solver used. On a CPU processor, the fast marching method~\cite{sethian1999fast} is used, which is made possible by the mathematical structure\footnote{Namely,  that~\eqref{eq_Update} is a non-decreasing function of $[(\ku-\ku_{ik})_+]^{1 \leq k \leq K}_{1 \leq i \leq I}$.} of the scheme as in  \cite{mirebeau2019hamiltonian,mirebeau2018fast,mirebeau2019riemannian}. This method works in a single pass way over the domain (the points of $\bM_h$ are successively \textsc{Accepted} one by one based on a priority queue) and has complexity $\cO(N \ln N)$, where $N = \#(\bM_h)$ is the number of discretization points.
The front propagation can be stopped as soon as the target point $\fp_1$ is reached. If a GPU accelerator is available, on the other hand, then a variant of the massively parallel fast iterative method~\cite{Jeong2008FIM} is used (which does not involve \textsc{Accepted} tags). The increased complexity of this approach, namely $\cO(N^{1+1/d})$ with $d=3$, and the stricter stopping criterion, namely global convergence of $u$, are more than compensated by the massive thread parallelism, resulting in a $15 \times$ or better speedup in applications. 

Given a source point $\fp=(p,\theta_p)$, \cref{Prop_SCC} ensures that the physical projections of the computed geodesic paths are simple closed and convex, provided that the fronts propagate within the subdomain $\tilde\bM$, hence \emph{not} across the obstacle $\{\theta_p\}\subset\bS^1$. 
In this case, the obstacle $\{\theta_p\}$ actually limits the total curvature $\tcurv$~\eqref{eq_TC} of those planar curves to $2 \pi$. 
For another option considered in~\cite{chen2021elastica}, let us mention that it is possible to track the total curvature $\tcurv$  by adapting a numerical method with an accumulation way~\cite{deschamps2001fast}, originally introduced to simultaneously estimate the weighted length and the Euclidean length of geodesic paths between the source point and any target point, without backtracking these paths. This allows to directly implement the constraint $\tcurv \leq 2 \pi$, rather than rely on the obstacle $\{\theta_p\}$. In practice, the two approaches produce similar numerical results, and in particular both guarantee the simplicity of the physical projection curves.

\section{Applications to Active Contours}
\label{sec_ACAPP}
In this section, We apply the proposed convexity-constrained geodesic  models to solve the active contour problems.

\subsection{Computation of Region-based Velocity}
\label{subsec_RegionRanders}
\subsubsection{Region-based Randers Geodesic Model}
We start from a typical active contour energy functional comprising of a  region-based homogeneity term $\Xi$ as well as a regularization term $\Phi$
\begin{equation}
\label{Eq_ACEnergy}
E(\gamma):= \mu\,\Xi(\gamma)+\Phi(\gamma),
\end{equation}
where $\mu>0$ is a positive constant and $\gamma\in\Lip([0,1],\Omega)$ is a closed curve. In general,  the term $\Phi(\gamma)$ can be defined as a weighted curve length w.r.t. a Riemannian metric, of the form 
\begin{equation*}
\Phi(\gamma)=\int_0^1 \|\dot\gamma(\varrho)\|_{\cM(\gamma(\varrho))}\,d\varrho.
\end{equation*}
The metric tensor $\cM$ is derived from the image gradients, in such way that $\|\dot x\|_{\cM(x)}=\sqrt{\<\dx,\cM(x)\dx\>}$ is low~\cite{chen2017global} when an edge passes by the point $x\in \Omega$ with a tangent direction that approximates the unit vector $\dx$.

\begin{figure*}[t]
\centering
\includegraphics[height=3.2cm]{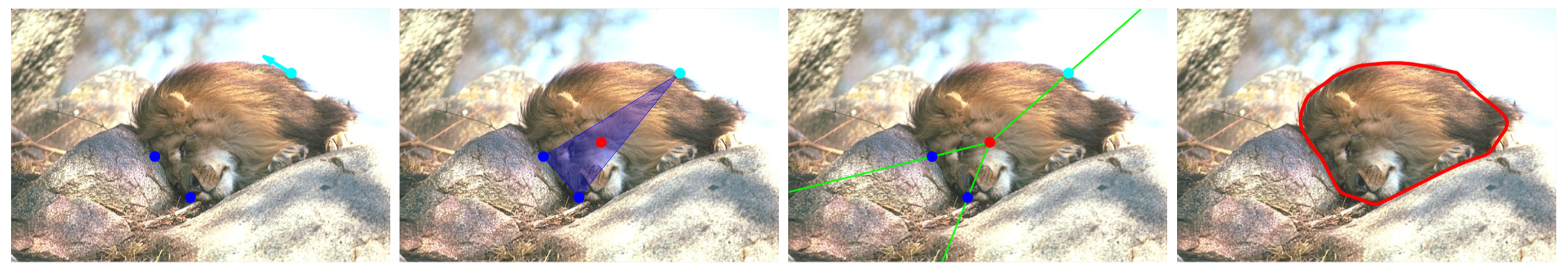}
\caption{An example for initialization derived from the landmark points-based annotation. The blue and cyan dots in columns $1$ to $4$ denote the landmark points. Column $1$: The cyan dot with the arrow indicate the source point $\fp$. Column $2$: the blue transparent region is the convex hull of all the landmark points, where the red dot represents its barycenter center. Column $3$: the green lines represent the additional obstacles. Column $4$: the segmentation contour (red line). }
\label{fig_BoundaryInteraction}				
\end{figure*}

\begin{figure*}[t]
\centering
\includegraphics[height=3.2cm]{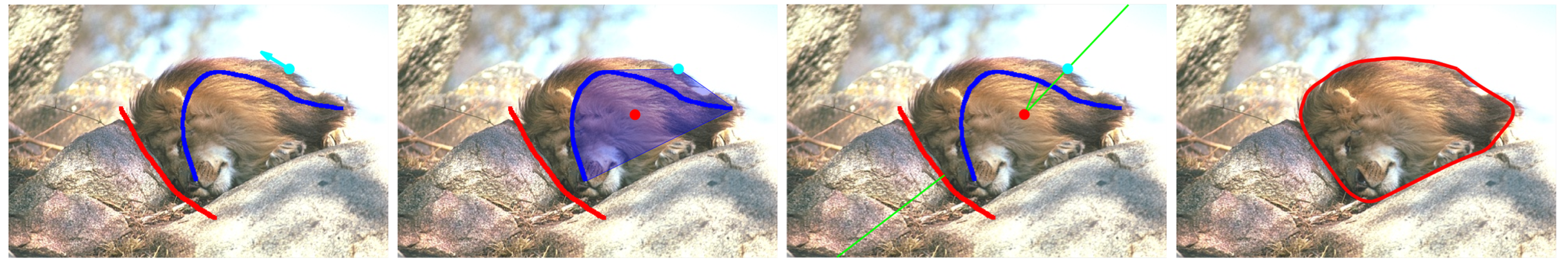}
\caption{An example for initialization derived from the scribbles-based annotation.  The blue and red curvilinear structures in columns $1$ to $4$ represent the interior and exterior scribbles, respectively.  Column $1$: The cyan dot and the arrow indicate the source point $\fp$. Column $2$:  the blue transparent region indicates the convex hull.  Column $3$: the green lines represent the additional obstacles. Column $4$: the segmentation contour (red line).}
\label{fig_RegionInteraction}				
\end{figure*}

The region-based homogeneity term $\Xi$, also referred to as the image appearance model,  measures the homogeneity of image features in each region. In this section, we take the region competition model~\cite{zhu1996region} with the Gaussian mixture model (GMM) as an example to formulate the region-based term 
\begin{equation}
\label{eqdef:Xi}
\Xi(\gamma)=\int_{R_1}\xi_1(x)dx+\int_{R_2}\xi_2(x)dx,
\end{equation}
where $R_1$ and $R_2$ are the regions inside and outside the closed curve $\gamma$. The scalar-valued functions $\xi_i:\Omega\to\bR$ ($i=1,2$) encode the image homogeneity features within each region $R_i$. We compute each $\xi_i$  using a GMM, for which the probability distribution function (PDF) $P_i(z;\Theta_i)$ is taken as a weighted sum of $N$ Gaussian PDFs. Let $f:\Omega\to\bR^d$ be an image, where $d=1$ (resp. $d=3$) implies that $f$ is a gray level (resp. color) image. Then one has 
\begin{equation}
 \xi_i(x)=-\log\big(P_i(f(x);\Theta_i)\big),~\forall x\in\Omega,
\end{equation}
where $\Theta_i$ are the parameters of the PDFs of the GMM. Moreover, the piecewise constant appearance model~\cite{chan2001active} is known as  an efficient variant of the GMM-based term $\Xi$. In this case, the function $\xi_i$ can be computed as $\xi_i(x)=\int_{R_i}	\|f(x)-\mathfrak{c}_i\|^2 dx$, where $\mathfrak{c}_i=(c_{i,1},c_{i,2},c_{i,3})\in\bR^d$  such that $c_{i,j}$ with $1\leq j\leq d$ stands for the mean intensity of the $j$-th image channel within each corresponding region $R_i$. 

In the region-based Randers geodesic model~\cite{chen2016finsler,chen2019region}, image segmentation is solved by minimizing the energy $E$ as formulated in~\eqref{Eq_ACEnergy}. 
A key ingredient of this model is to express, using Stokes theorem, the region integral \eqref{eqdef:Xi} as a boundary integral (plus a constant). As a result, the energy~\eqref{Eq_ACEnergy} is reformulated as a weighted anisotropic curve length,
%A key ingredient for this model is to express, using Stokes theorem, the energy~\eqref{Eq_ACEnergy} as a weighted curve length, 
i.e. $E(\gamma) = \tilde\Phi(\gamma) + \sigma$, where $\sigma$ is a scalar value independent of $\gamma$, and where
\begin{equation}
\label{eq_RandersAppro}
\tilde\Phi(\gamma)=\int_0^1 \|\dot\gamma(\varrho)\|_{\cM(\gamma(\varrho))}+\mu\<\varpi(\gamma(\varrho)),\dot\gamma(\varrho) \> d\varrho.
\end{equation}
The vector field $\varpi:\bR^2\to\bR^2$ is defined over an open bounded region $U\subset\Omega$, and is obtained as the solution of the linear PDE
\begin{equation}
\label{eq_ConVec}
\min~\int_{\bR^2}\|\varpi(x)\|^2 d x,~s.t.~
\curl\varpi = (\xi_1-\xi_2) \chi_U,
\end{equation}
where $\chi_U:\bR^2\to\{0,1\}$ is the characteristic function of the subdomain $U$. The solution to the linear problem \eqref{eq_ConVec} can be obtained by convolution of the r.h.s.\ $(\xi_1-\xi_2) \chi_U$ with a suitable kernel~\cite{chen2019region}. 
In a variant of \eqref{eq_ConVec}, used in the experiments, the objective function is replaced with $\int_{U}\|\varpi(x)\|^2 d x$, so that a staggered grid finite difference method can be used over the domain $U$, which provides less guarantees but often yields better numerical behavior \cite{chen2019region}.

The weighted length~\eqref{eq_RandersAppro} is an instance of Randers geometry, defined by a non-symmetric metric
\begin{equation}
\label{eq_RegionRandersMetric}
\kR(x,\dx)=\|\dx\|_{\cM(x)}+\mu\<\varpi(x),\dx\>.
\end{equation}
It is proven in~\cite{chen2019region} that the metric $\kR$ is positive definite provided the region $U$ is sufficiently small and the first order term is defined by Eq.~\eqref{eq_ConVec}. During the curve evolution, $U$ should be understood as the search space for the evolving curves.

\begin{figure}[t]
\centering
\includegraphics[height=4cm]{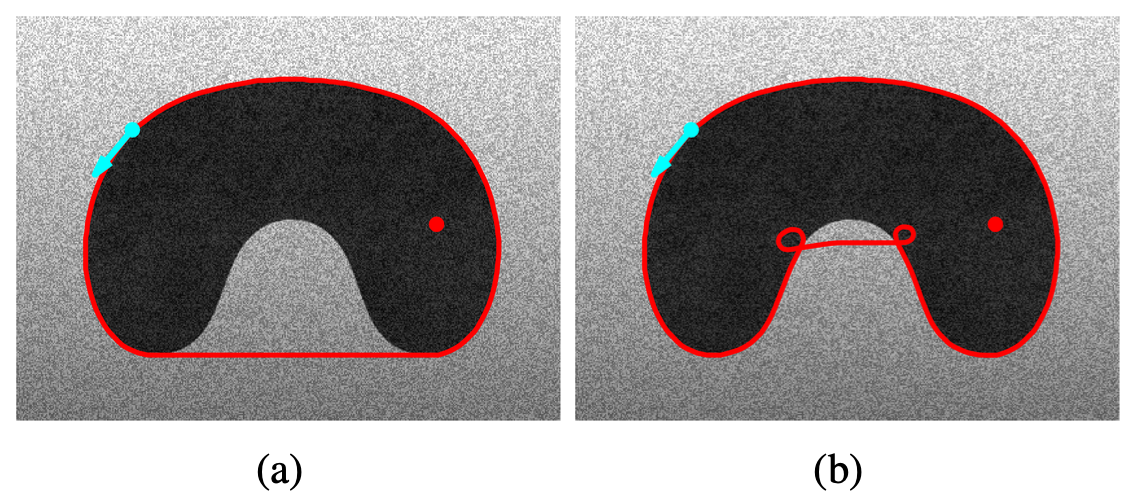}
\caption{\textbf{a} and \textbf{b}: Geodesic paths associated with the Elastica-Convexity model, which are generated respectively with  and  without the angular obstacle $\{\theta_p\}$ in the angular domain $\bS^1$. The red dots are the point $z$ and the cyan dots and arrows indicate the source point $\fp=(p,\theta_p)$.}
\label{fig_Simplicity}				
\end{figure}

\begin{figure*}[t]
\centering
\begin{minipage}[t]{0.225\textwidth}
\includegraphics[width=4.2cm]{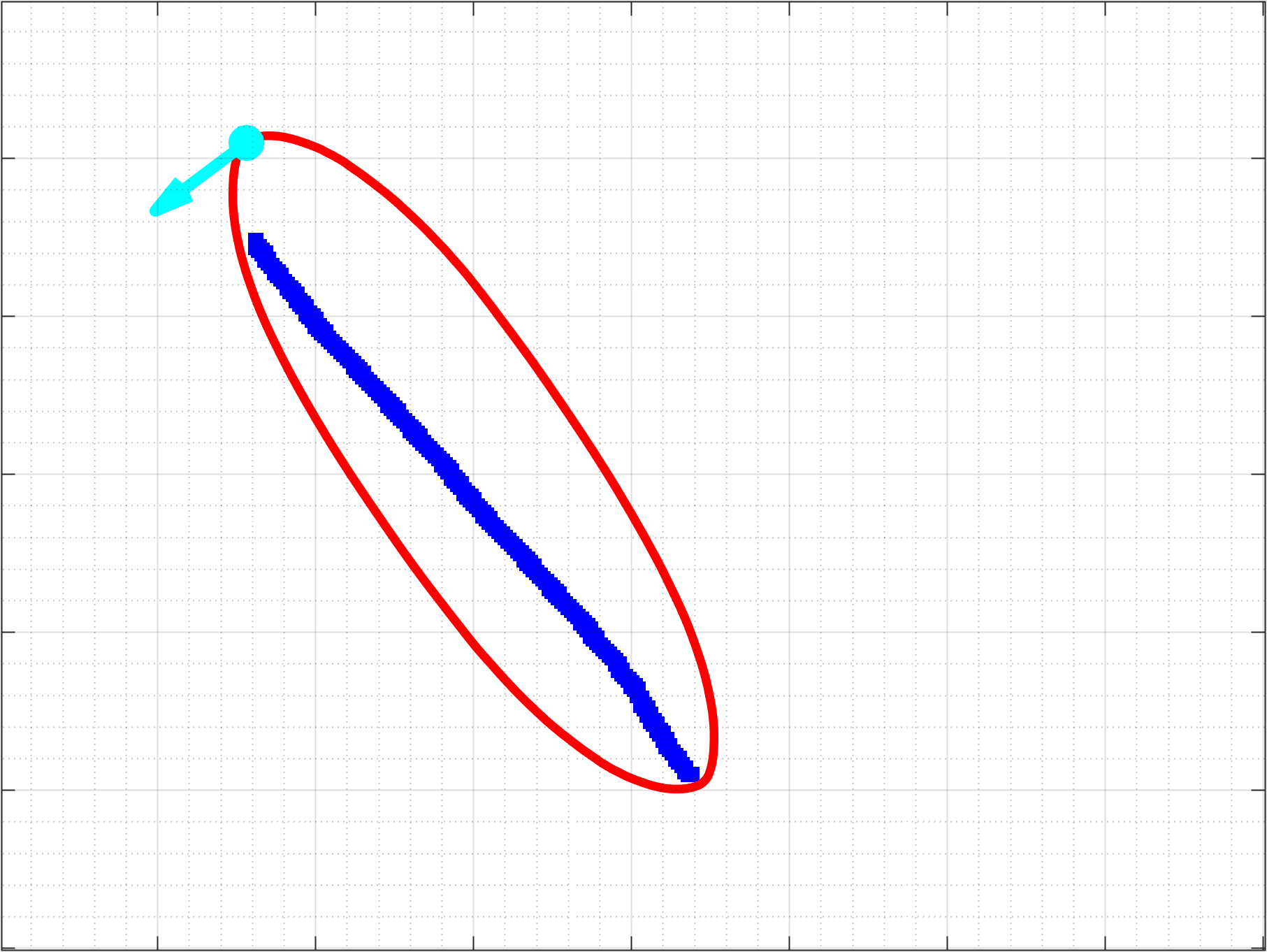}\vspace{1mm}
\includegraphics[width=4.2cm]{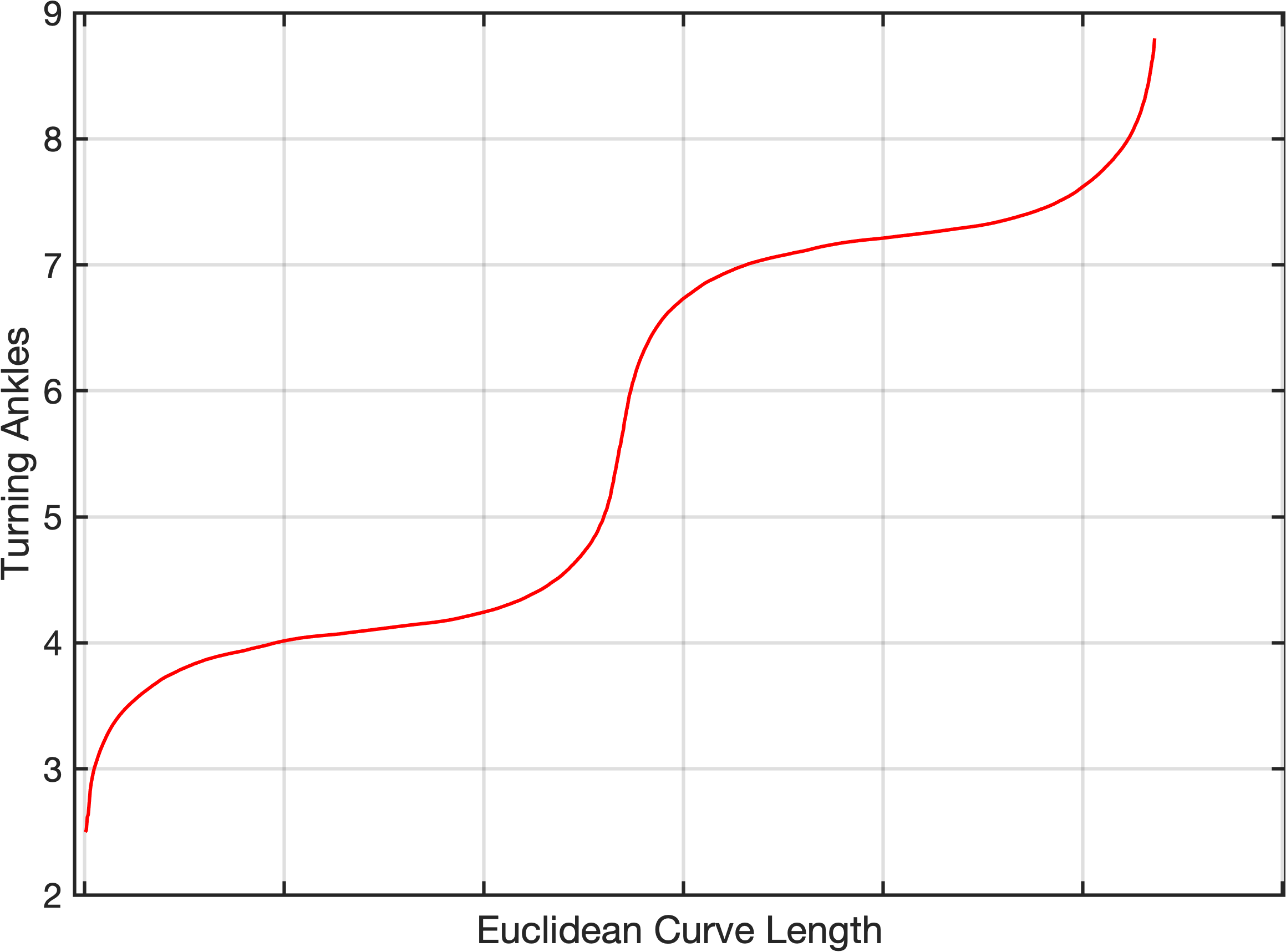}\vspace{1mm}
\centering\footnotesize{RSF-Convexity}\vspace{1mm}
\end{minipage}
\begin{minipage}[t]{0.225\textwidth}
\includegraphics[width=4.2cm]{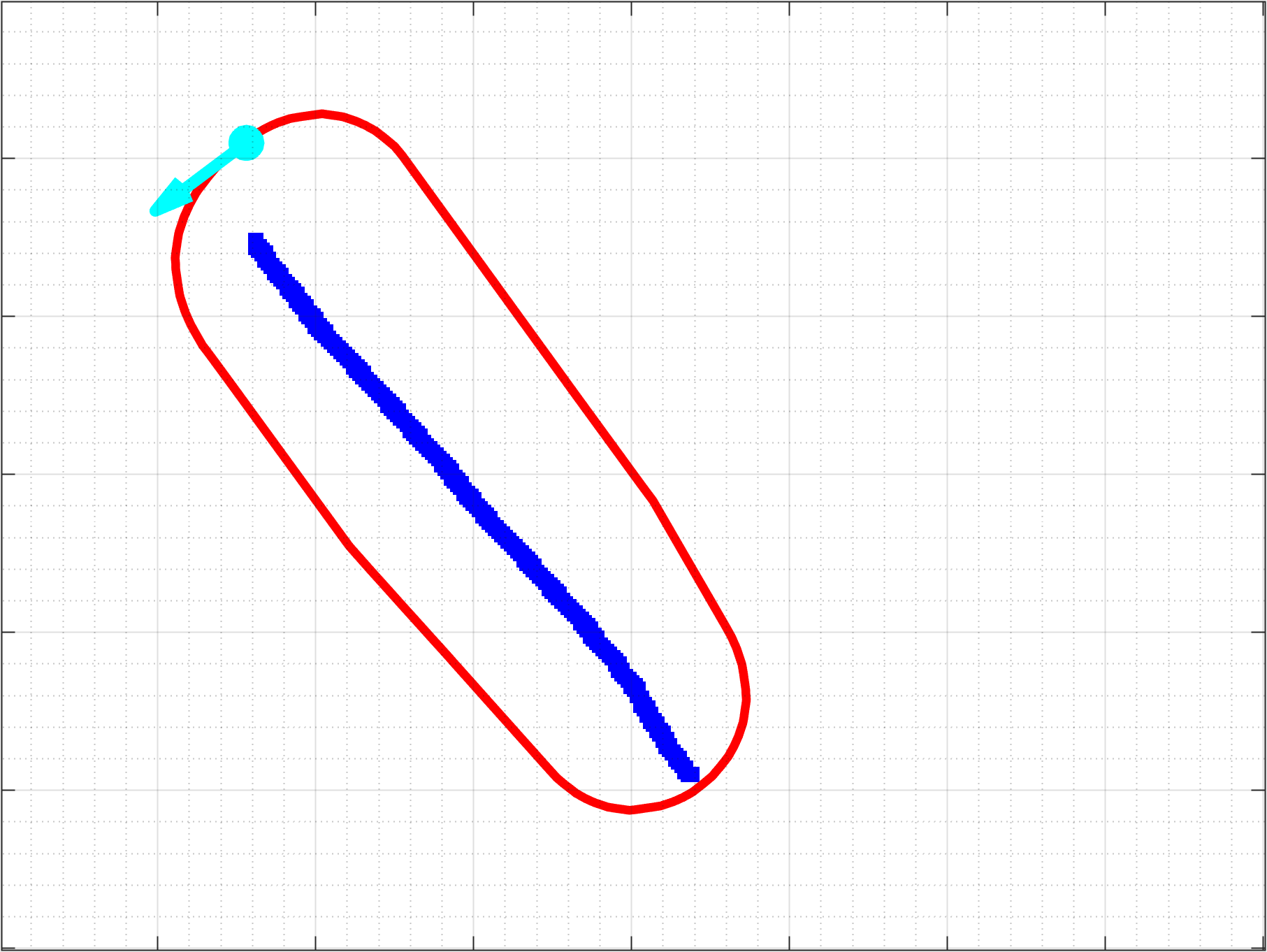}\vspace{1mm}
\includegraphics[width=4.2cm]{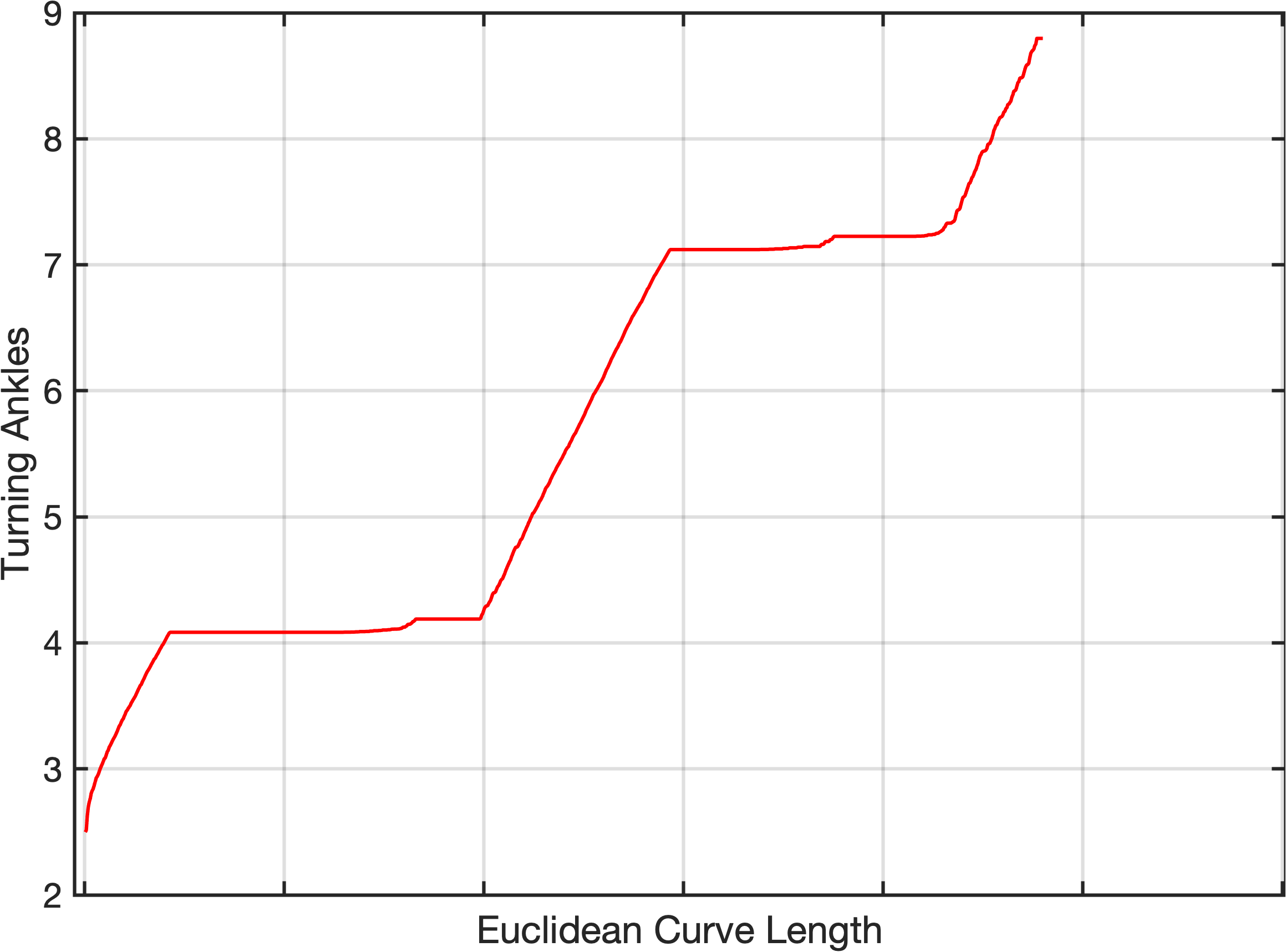}\vspace{1mm}
\centering\footnotesize{Dubins-Convexity}\vspace{1mm}
\end{minipage}
\begin{minipage}[t]{0.225\textwidth}
\includegraphics[width=4.2cm]{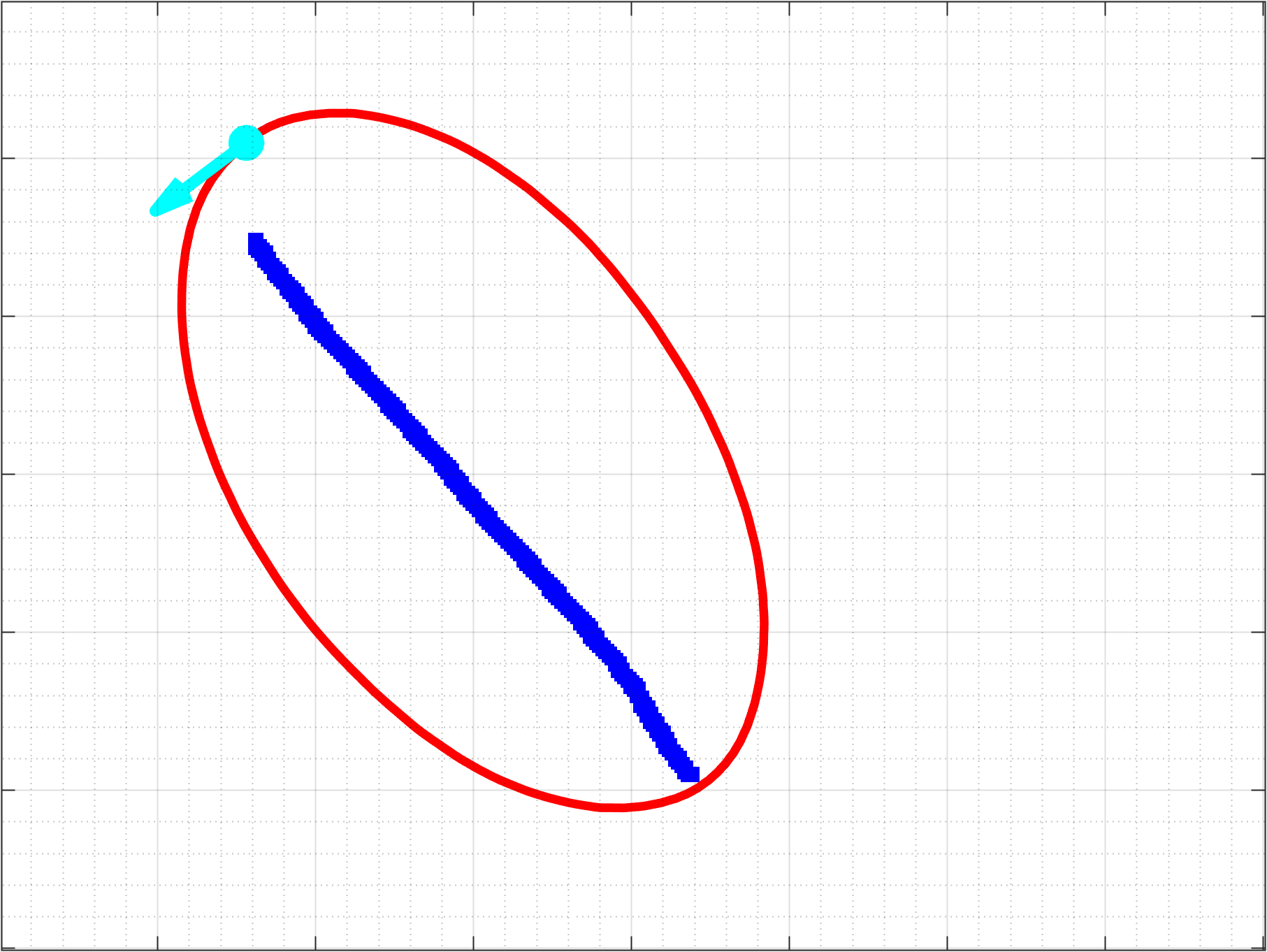}\vspace{1mm}
\includegraphics[width=4.2cm]{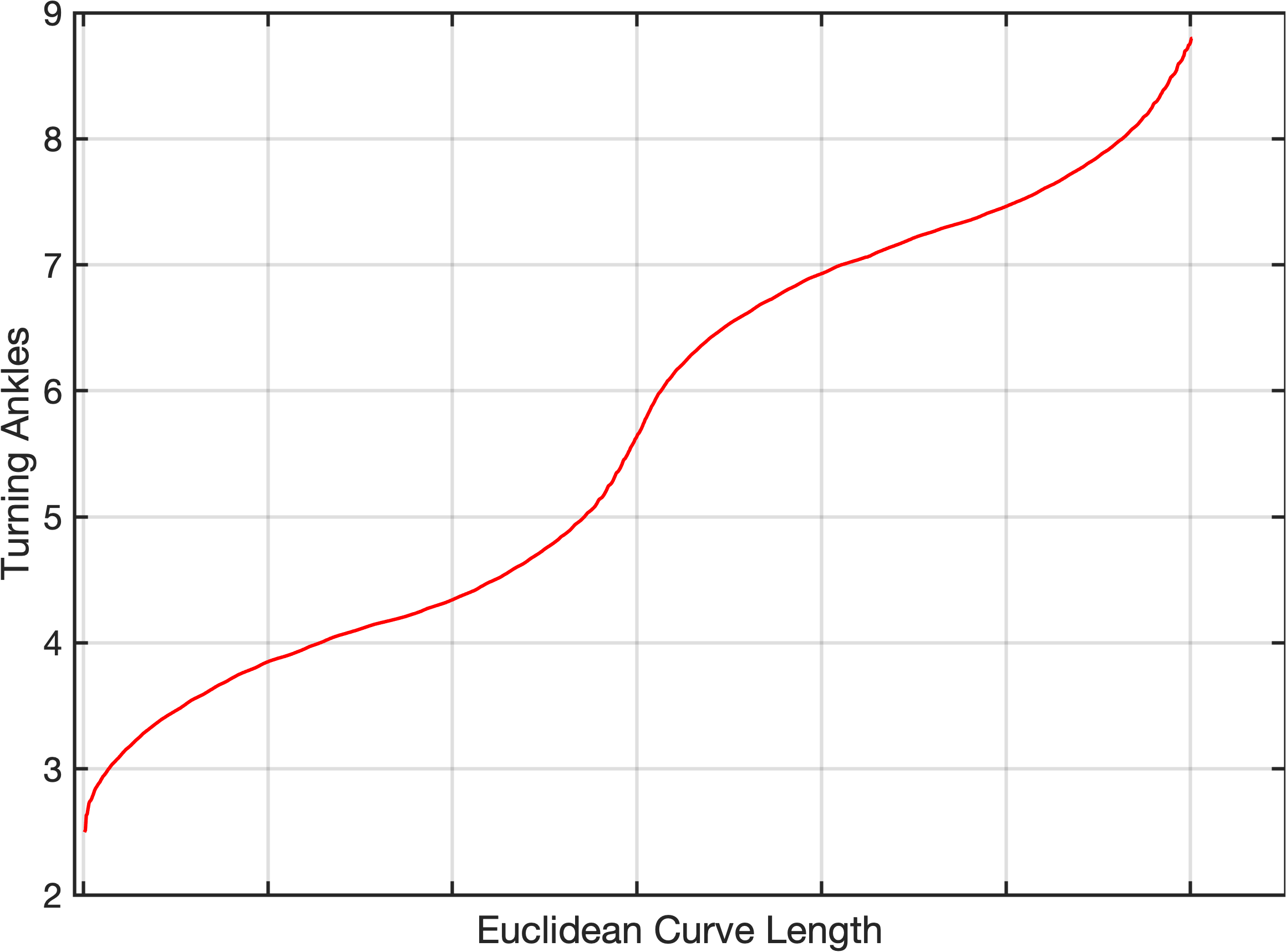}\vspace{1mm}
\centering\footnotesize{Elastica-Convexity}\vspace{1mm}
\end{minipage}
\begin{minipage}[t]{0.225\textwidth}
\includegraphics[width=4.2cm]{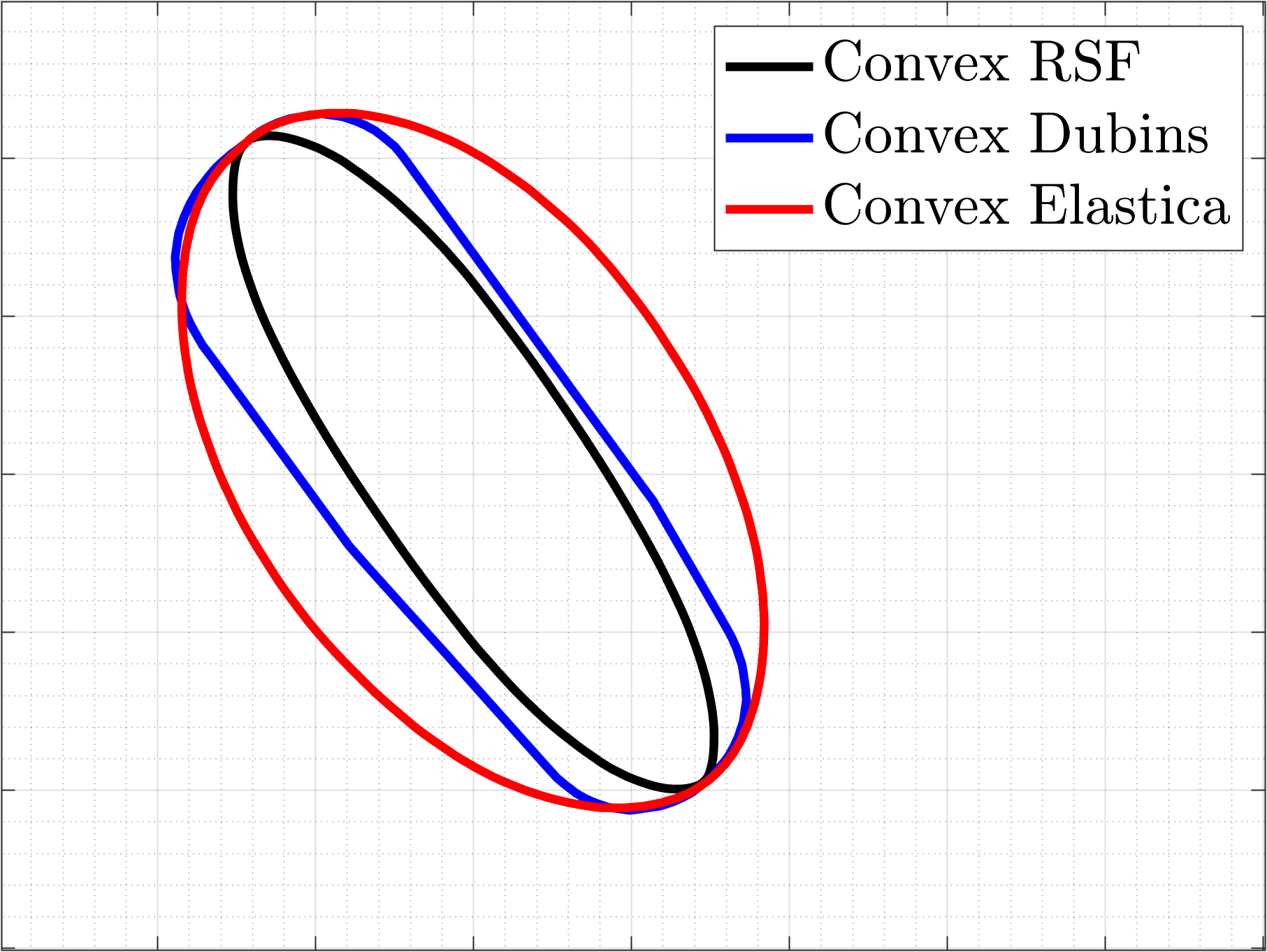}\vspace{1mm}
\end{minipage}
\caption{\textbf{Top}: The geodesic paths (red lines) in columns $1$ to $3$ are respectively derived from the proposed RSF-Convexity, Dubins-Convexity and Elastica-Convexity  models. The blue curvilinear structures are the foreground scribbles, and the cyan dots with arrows represent the source point. \textbf{Bottom}: The plots of the turning angles of the corresponding orientation-lifted geodesic paths.}
\label{fig_Uniform}				
\end{figure*}

\subsubsection{Orientation-lifted Velocity}
The weighted curve length~\eqref{eq_RandersAppro} can be theoretically interpreted in the framework of orientation lifting, by choosing the cost $\psi(x,\theta) = \kR(x,\dfnt)$. This leads to the possibility of integrating the region-based homogeneity features and curvature regularization for solving the active contour problems, using a sufficiently small $U$ to ensure the positivity of $\psi(x,\theta)$. However, such an interpretation is not what we do in this paper.  In contrast, we construct the velocity $\psi$ as an exponential cost of  the components of the metric $\kR$. Specifically, we consider 
\begin{equation}
\label{eq_Velocity}
\psi(x,\theta)=
\begin{cases}
\exp\left(\alpha\,\tilde\psi(x,\theta)\right),&\forall x\in U,\\
\infty,&\text{otherwise}, 
\end{cases}
\end{equation}
where $\alpha>0$ is a constant and where the function $\tilde\psi$ is defined as
\begin{equation}
\tilde\psi(x,\theta):=\frac{\|\dfnt\|_{\cM(x)}}{\sup_{(y,\vartheta)\in\bM}\|\dfn_\vartheta\|_{\cM(y)}}+\frac{\mu\<\varpi(x),\dfnt\>}{\sup_{y\in\Omega}\|\varpi(y)\|}.
\end{equation}
This construction of  $\psi$ proves to be very efficient in practice, although the connection with Eqs.~\eqref{Eq_ACEnergy} and \eqref{eq_RandersAppro} is partly lost.

\subsection{Convexity-constrained Active Geodesic Paths}
\label{Subsec_ConvexAC_Ini}
We apply the proposed convexity-constrained models for interactive image segmentation in conjunction with a curve evolution scheme. The evolving curves are the physical projections of the orientation-lifted geodesic paths from the models with convexity shape prior. When the evolution stabilizes, the target boundaries can be delineated by the obtained physical projections. Overall, the considered interactive segmentation algorithm can be divided into two  steps: (i) establishing extra constraints for computing geodesic paths under user-specified annotations, and (ii) evolving the geodesic paths using the proposed geodesic models. In the following, we will present the details for those steps. 

\subsubsection{Extra Constraints for Geodesic Paths}
\label{subsubsec_TightSearchSpace}
\noindent\emph{Scribbles-based Annotations}.
Scribbles are very often taken as initial annotations in interactive segmentation algorithms. Here we propose a way for geodesic-based segmentation, which allows to leverage foreground and background scribbles as extra constraints. In our model, each scribble is regarded as a subregion of the image domain $\Omega$. We can randomly sample a point $x_F$ from each foreground scribble $F\subset\Omega$. The union $([z,x_F]\cup{F})\times\bS^1$ serves as an obstacle such that no curve is allowed to passed through it. 
At the same time, one can choose a point $x_B$ from each background scribble $B\subset\Omega$, yielding a segment $[x_B,q]$ where $q\in\partial\Omega$ is a point subject to  $(z-x_B) \propto (x_B-q)$. Similarly to $([z,x_F]\cup{F})\times\bS^1$, the union $([x_B,q]\cup{B})\times\bS^1$ also forms an obstacle in the search space of geodesic paths.

When numerically computing the geodesic distance values by the HFM method,  any point $\fy$ will be removed from the stencil $\cS(\fx)$ if the segment $[\fx,\fy]$ intersects the obstacles generated by the scribbles and the point $z$.

\noindent\emph{Landmark points-based Annotations}.
In the context of interactive image segmentation, boundary-based annotations are usually carried out by a family of landmark points $x_k\in\Omega$, indexed by $1\leq k \leq \mathfrak{K}$, such that each point $x_k$ is placed at the target boundary. In contrast to traditional approaches~\cite{mille2015combination}, we \emph{do not} impose any order to the points $x_k$. 

Let $\Re_{z}(x_k)$ be a ray line or a half straight line emanating from point $z$ and passing through point $x_k$, and let $q_k$ be the intersection point between  $\Re_{z}(x_k)$ and the target boundary. It is easy to see that the orientation-lifted geodesic paths do not pass through the wall $[z,x_k[\times\bS^1$ and  $]x_k,q_k]\times\bS^1$, due to the convex assumption on their physical projections.  Note that for each index $k$, neither $[z,x_k[$ nor $]x_k,q_k]$ involves the landmark point $x_k$. Numerically, when computing geodesic distances by the HFM method, the walls $[z,x_k[\times\bS^1$ and  $]x_k,q_k]\times\bS^1$ are used to refine the stencils $\cS$. Specifically, a point $\fy\in\cS(\fx)$ should be excluded from $\cS(\fx)$, if the segment  $[\fx,\fy]$ satisfies that $[\fx,\fy]\cap ([z,x_k[\times\bS^1) \neq \emptyset$, or $[\fx,\fy]\cap(]x_k,q_k]\times\bS^1) \neq \emptyset$. As a result, the physical projection of any closed  geodesic paths $\cG$, subject to $\cG(0)=\cG(1)=\fp$, will pass through all landmark points $x_k$ for $1\leq k \leq \mathfrak{K}$.

\noindent\emph{Automatic Detection of the point $z$}.
The point $z$ used for defining the search space of geodesic paths can be  automatically derived from the  user annotations. By the convexity assumption of the target regions, it is natural to  set the point $z$ as the barycenter of the convex hull of the set $\{p,\,x_k;k=1,\cdots, \mathfrak{K}\}$ (resp. $F\cup\{p\}$ with $F$ being all the foreground scribbles)  for the landmark points-based (resp. scribbles-based) annotation way, where $p$ is the physical position of the source point $\fp=(p,\theta_p)$. In this way,  the detected point $z$ must lie at the interior of the target region.

\subsubsection{Active Geodesic Path Evolution Procedure}
\label{subsubsec_ActiveGeos}
In this section, we incorporate the  convexity-constrained geodesic models into the region-based Randers geodesic model~\cite{chen2019region}, such that the image segmentation can preserve the advantages of user annotations, curvature regularization and convexity shape prior.  

Using a curve evolution scheme, the goal is to generate a sequence of closed geodesic paths $\{\cG_j\}_{j\geq 0}$, each of which lies at the space $\bM$ and solves the problem~\eqref{eq_optim} and is such that the physical projection $\sC_j$ is a simple closed and convex planar curve, obeying $\sC_j(\varrho)\in U_j,\forall \varrho\in [0,1]$. In this way, the search space $U_j$ at the $j$-th iteration is used to find the solution to the problem~\eqref{eq_ConVec}, in order to update the velocity $\psi_j$ through Eq.~\eqref{eq_Velocity}. Moreover, as in the region-based geodesic model~\cite{chen2019region}, we choose the search space $U_j$ as a tubular neighbourhood of $\sC_{j-1}$. 

Recall that the initialization for the proposed convexity-constrained geodesic models requires a point $z\in R$ with $R\subset\Omega$ being the target region, and a point $\fp=(p,\theta_p)\in\bM$ such that $p\in\partial{R}$, as described in Section~\ref{subsec_ExtendCG}. In addition to these points,  we take into account two types of interaction ways for building the sequence $\{\cG_j\}_{j\geq 0}$, in order to accommodate complicated situations. Note that the numerical scheme for computing each orientation-lifted geodesic path  $\cG_j$ is presented in Section~\ref{sec_HFM}.

\noindent\emph{Building the initial curves}.
The initial curve $\cG_0=(\sC_0,\eta_0)$ should obey that (i) $\cG_0(0)=\fp$, and (ii) the physical projection $\sC_0$ is simple closed and convex. In order to simplify the initialization process of the proposed segmentation method, we construct the initial curve  $\cG_0$ as a closed minimal path by solving the problem~\eqref{eq_optim}, subject to the user-provided annotations discussed above. 

The  image gradients-based features are independent to the evolving geodesic curves $\cG_j$. As a consequence,  such an initial curve $\cG_0$ can be produced using the edge-based features only. We denote by $\psi_{\rm edge}$ such an edge-based velocity, reading as 
\begin{equation}
\label{eq_EdgeVelocity}
\psi_{\rm edge}(x,\theta)=\exp\left(\frac{\alpha\,\|\dfnt\|_{\cM(x)}}{\sup_{(y,\vartheta)\in\bM}\|\dfn_\vartheta\|_{\cM(y)}}\right),
\end{equation}
which is independent to the tubular neighbourhood $U_j,\,\forall j\geq 0$.

\begin{figure*}[t]
\centering
\includegraphics[width=0.95\textwidth]{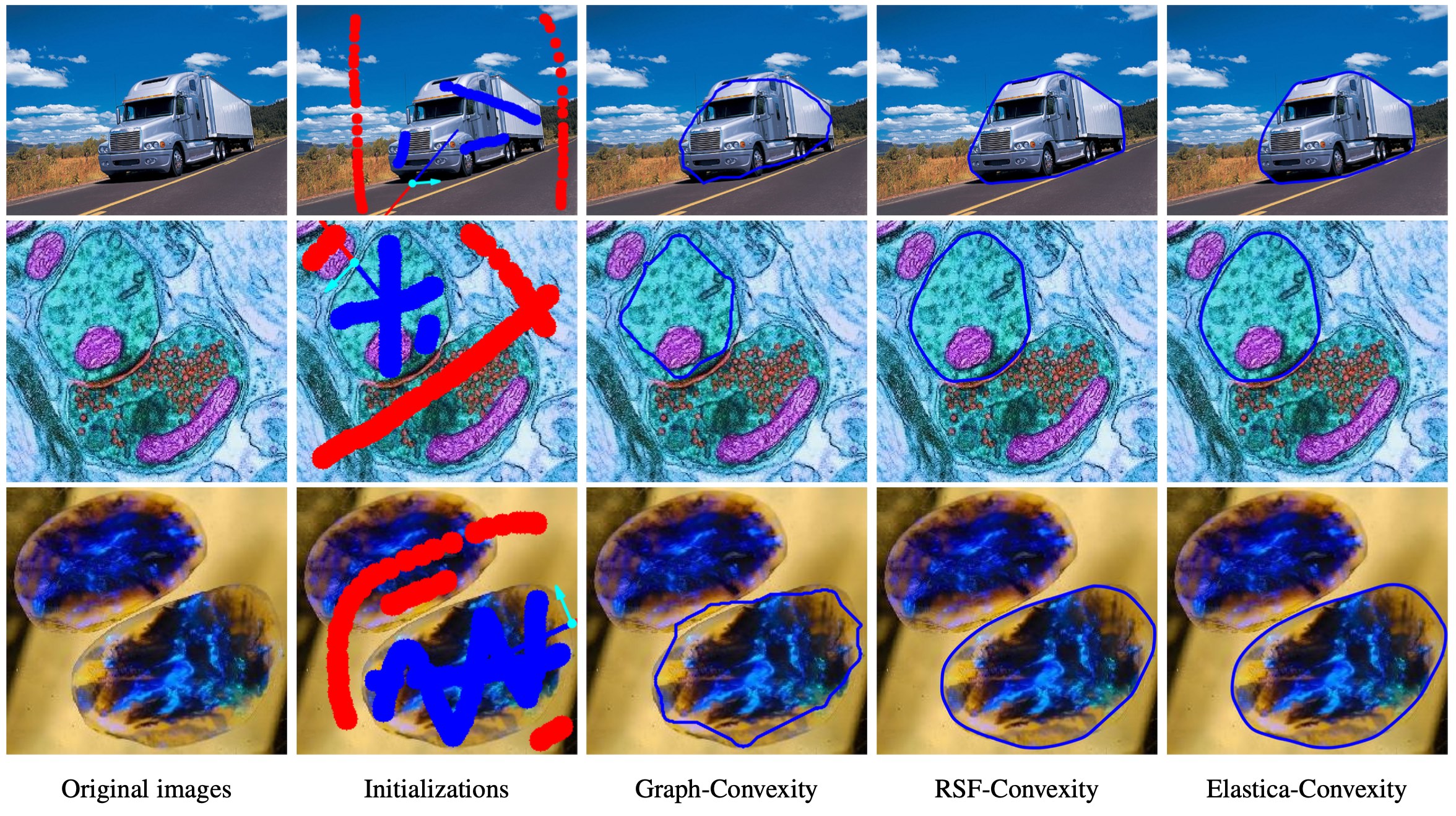}
\caption{Qualitative comparison results with the Graph-Convexity model. Column $1$: Original images. Column $2$: Initializations. The blue and red curvilinear structures denote the scribbles inside and outside the target regions. The source points are visualized via the cyan dots and the respective arrows. Columns $3$ to $5$: Segmentation results from different models. }
\label{fig_TRC}	
\end{figure*}

\begin{figure*}[t]
\centering
\includegraphics[width=0.95\textwidth]{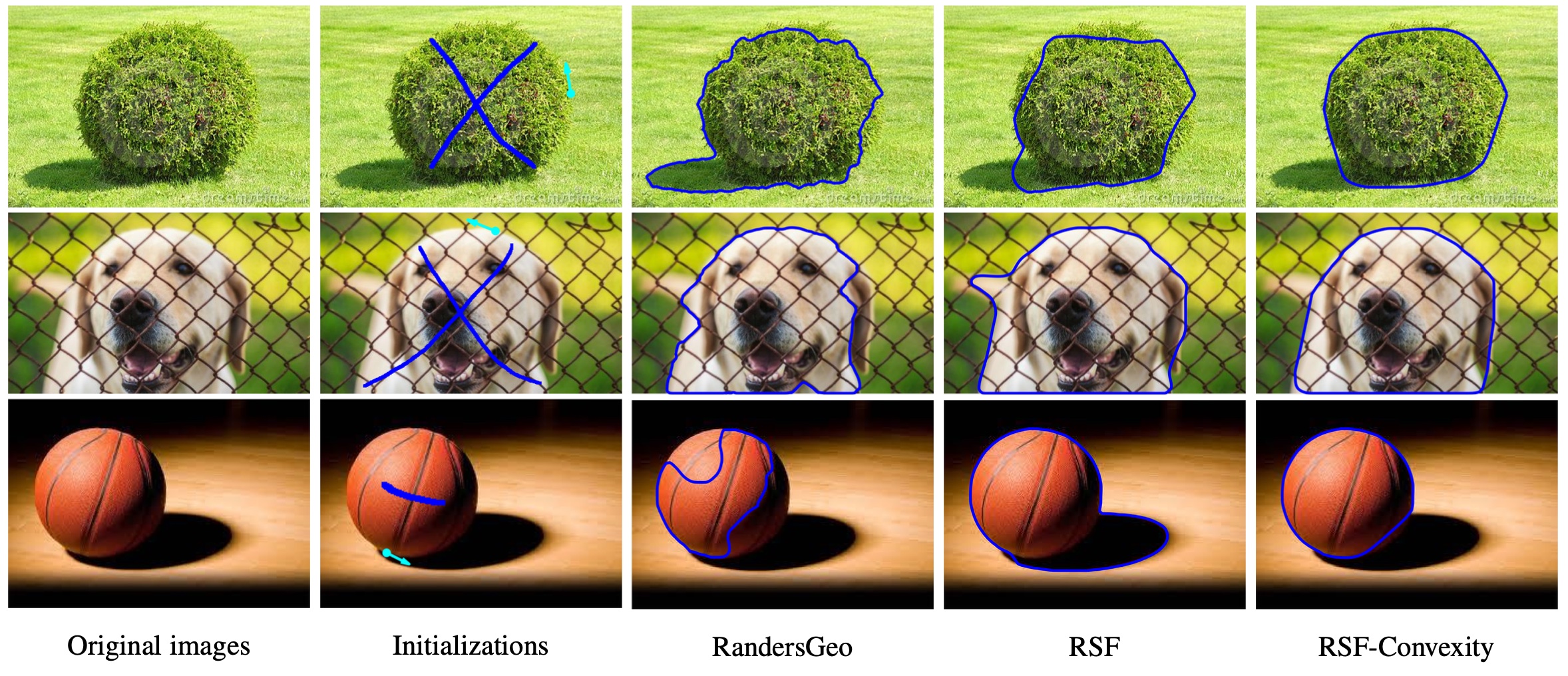}
\caption{Qualitative comparison results with state-of-the-art geodesic models. Columns $1$ and $2$ show the original images and the corresponding annotations. The blue structures in column $2$ denote the scribbles inside  the target regions. The source points are visualized via the cyan dots and the respective arrows. Columns $3$ to $5$: Image segmentation results from different models.}
\label{fig_CompGeos}	
\end{figure*}

\begin{figure*}[t]
\centering
\includegraphics[width=0.95\textwidth]{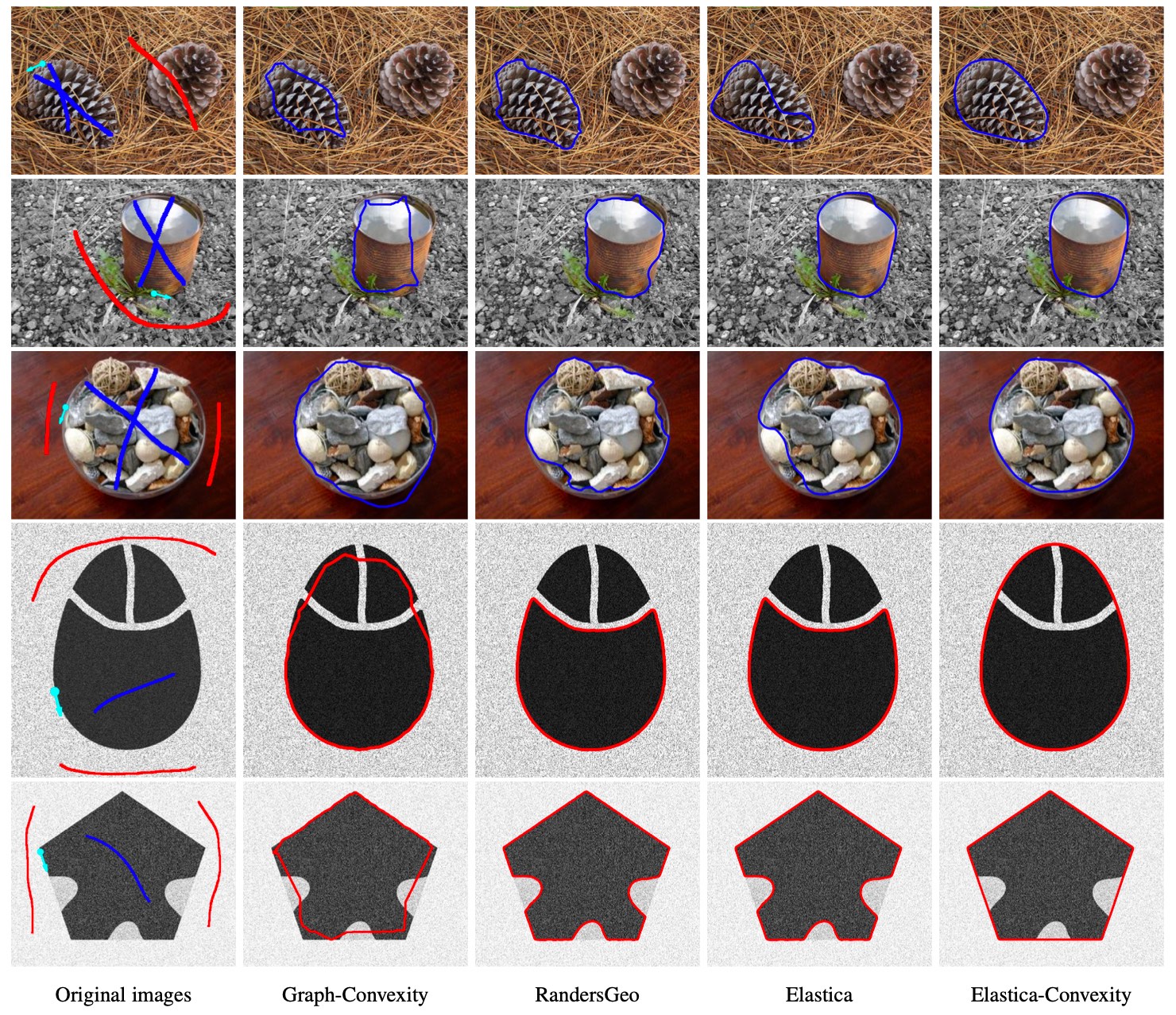}
\caption{Qualitative comparison results with state-of-the-art geodesic models. Columns $1$  show the original images with the initial annotations. The blue and red curvilinear structures denote the scribbles inside and outside the target regions. The cyan dots together with the respective arrows indicate the source points. Columns $2$ to $5$: Image segmentation results from different models.}
\label{fig_Synthetic_Weizzman}	
\end{figure*}

\section{Experimental Results}
\label{sec_Exp}  
In this section, the experiments are dedicated to illustrate the advantages of exploiting the integration of the convexity shape prior and the curvature penalty for image segmentation. After the study on the properties of the proposed models,  we focus on the qualitative and quantitative comparisons against the graph-based segmentation model with convexity constraint (Graph-Convexity)~\cite{gorelick2016convexity}, the region-based Randers geodesic (RandersGeo) model~\cite{chen2019region} and the classical curvature-penalized models involving the RSF~\cite{duits2018optimal}, Dubins~\cite{mirebeau2018fast} and EM elastica models~\cite{chen2017global}. 
Moreover, the wall $\ray\times\bS^1$ (resp. the wall $\ray$) is used to track closed geodesic paths for the classical curvature-penalized models (resp. the RandersGeo model). The intersection tests with the wall $\{\theta_p\}$ are removed for the classical curvature-penalized geodesic models. 

\subsection{Influence of the Obstacle in the Angular Domain}
\label{subsec_InfluenceAngularWall}
Let $\fp=(p,\theta_p)\in\bM$ be a source point. In the HFM method,  the obstacle $\{\theta_p\}$ in the angular domain $\bS^1$ acts as a wall to restrict the fronts of the geodesic distances to propagate within the obstacle-free domain $\tilde\bM$, see Eq.~\eqref{eqdef_obs}. This amounts to an implicit constraint on the total curvature the physical projections of the target geodesic paths, so as to guarantee the simplicity of those planar curves. In Fig.~\ref{fig_Simplicity}, we conduct two tests using the  Elastica-Convexity model, in order to exhibit  the crucial influence of this obstacle $\{\theta_p\}$ in computing geodesic paths whose physical projections are simple closed and convex. In this figure, the source point $\fp=(p,\theta_p)$ is visualized via a pair involving a dot and an arrow of cyan color. The red lines are the physical projections of the extracted orientation-lifted geodesic paths.  Figs.~\ref{fig_Simplicity}a and~\ref{fig_Simplicity}b respectively demonstrate the results of the Elastica-Convexity model with and without the use of the angular obstacle $\{\theta_p\}$ during the computation of geodesic distances. In Fig.~\ref{fig_Simplicity}b, we can see that the absence of $\{\theta_p\}$ leads to self-crossings in the physical projection curve, whose total curvature $\tcurv$ exceeds the limit $2\pi$. Accordingly, the convexity property of the physical projection is lost. In contrast, the physical projection curve in Fig.~\ref{fig_Simplicity}a, where the associated geodesic distance map is numerically estimated in $\tilde\bM$, is simple and convex,  proving the importance and necessity of the obstacle $\{\theta_p\}$ in our convexity-constrained geodesic models.

\subsection{Properties of the Proposed Geodesic Models}
\label{subsec_Properity}
In Fig.~\ref{fig_Uniform}, we illustrate the qualitative differences between the proposed convexity-constrained geodesic models, by choosing a constant velocity $\psi\equiv1$ and setting the parameter $\beta=4$, see Eq.~\eqref{eq_LengthConvexity}, for each tested model. At the top row of columns $1$ to $3$, the red lines represent the physical projections of the  geodesic paths, computed from the RSF-Convexity, Dubins-Convexity and Elastica-Convexity models, respectively.  
At the bottom row, we plot the turning angles (i.e. the angular component $\eta$) of the orientation-lifted geodesic paths $\Gamma=(\gamma,\eta)$, each of which is parameterized by its Euclidean curve length. The physical projections are smooth and convex, and surround the scribbles indicated by the blue color. More specifically, the physical projection associated to the RSF-Convexity model allows the presence of high curvature values, which also can be observed from the plot of the turning angles shown at the bottom row of column $1$. At the top row of column $2$, the physical projection from the Dubins-Convexity model can be approximately divided into straight segments and parts of circles, which is also discussed in~\cite{mirebeau2019hamiltonian}. Such an observation can be verified by the plot of turning angles, the graph of which is the concatenation of straight segments of different slopes. Moreover,  we notice that the plots of the turning angles for all the convexity-constrained models indicate  non-decreasing curvature values, consistently with~\cref{def_Convex}. Finally, the qualitative differences between the proposed geodesic models are quite similar to those between the original curvature-penalized geodesic models~\cite{mirebeau2019hamiltonian}, due to their relevance.

\subsection{Qualitative Comparisons}

We declare the values of the parameters used in the introduced convexity-constrained models. Specifically, we choose the parameters $\alpha\in\{3,4\}$ and $\mu\in\{0.1,1\}$ to perform the computation of the velocity~\eqref{eq_Velocity} for all models, and  adopt the parameter $\beta\in\{1,1.2\}$ for the relative importance of the curvature regularization in the RSF-Convexity  and Elastica-Convexity models. For the Dubins-Convexity model, the parameter $\beta$ defines a hard constraint on the curvature of the physical projections of the target geodesic paths, as formulated in Eq.~\eqref{eq_CDubins}. This lack of flexibility means that $\beta$ must be finely tuned for each individual image.

\begin{table*}[t]
\centering
\caption{Quantitative comparison of different models by the average and standard deviation values of Jaccard scores on $5$ images shown in Fig.~\ref{fig_Synthetic_Weizzman}.}
\label{table_Syntheic}
\setlength{\tabcolsep}{6pt}
\renewcommand{\arraystretch}{1.3}
\begin{tabular}{c c cc c cc c cc c cc c cc c}
\shline
IMAGES& &\multicolumn{2}{c}{Row $1$} & & \multicolumn{2}{c}{Row $2$} & &\multicolumn{2}{c}{Row $3$} & &\multicolumn{2}{c}{Row $4$}& &\multicolumn{2}{c}{Row $5$} &\\
\cline{1-1}\cline{3-4}  \cline{6-7} \cline{9-10}\cline{12-13}  \cline{15-16} 
&   & Ave.   & Std.  & & Ave.   & Std.    & & Ave.   & Std.   & & Ave.   & Std.   && Ave.   & Std. \\
\hline
Graph-Convexity&       &$57.85\%$ &$0.04$   &  &$68.74\%$  &$0.05$  & &$88.59\%$ &$0.01$ &  &$92.38\%$ &$0.01$ & &$86.04\%$  &$0.02$\\
RandersGeo&            &$86.23\%$ &$0.02$   &  &$88.15\%$  &$0.02$  & &$88.38\%$ &$0.01$ &  &$72.56\%$ &$0.06$ & &$89.41\%$  &$0$\\
RSF&                   &$88.72\%$ &$0.00$   &  &$90.21\%$  &$0.02$  & &$92.86\%$ &$0.01$ &  &$74.84\%$ &$0.10$ & &$89.90\%$  &$0.01$\\
Dubins&                &$49.90\%$ &$0.06$   &  &$86.83\%$  &$0.02$  & &$85.07\%$ &$0.02$ &  &$74.67\%$ &$0.10$ & &$89.77\%$  &$0$\\
Elastica&              &$89.16\%$ &$0.01$   &  &$91.11\%$  &$0.02$  & &$93.04\%$ &$0.01$ &  &$78.21\%$ &$0.15$ & &$89.91\%$  &$0.01$\\
RSF-Convexity&         &$89.08\%$ &$0.02$   &  &$93.41\%$  &$0.01$  & &$95.63\%$ &$0.00$ &  &$99.60\%$ &$0$    & &$99.50\%$  &$0$\\
Dubins-Convexity&      &$85.96\%$ &$0.02$   &  &$87.60\%$  &$0.03$  & &$89.07\%$ &$0.01$ &  &$99.59\%$ &$0$    & &$99.58\%$  &$0$\\
Elastica-Convexity&    &$90.37\%$ &$0.02$   &  &$93.48\%$  &$0.01$     & &$95.49\%$ &$0$ &  &$99.63\%$ &$0$ & &$99.49\%$  &$0$\\
\shline
\end{tabular}
\end{table*}
%\diagbox{Models}{$\JS$}

\begin{table*}[t]
\centering
\caption{Quantitative comparison of different models (applied with Scribbles- and Landmark points-based annotations) by the average and standard deviation values of the mean accuracy scores on the Convexity dataset.}
\label{table_HybridAnnotations}
\setlength{\tabcolsep}{6pt}
\renewcommand{\arraystretch}{1.3}
\begin{tabular}{l ccc c ccc c ccc ccc}
\shline
Annotations& &\multicolumn{2}{c}{Scribbles} & &\multicolumn{2}{c}{3 Landmark Points}&   &\multicolumn{2}{c}{4 Landmark Points}& &\multicolumn{2}{c}{5 Landmark Points}&\\
\cline{1-1}\cline{3-4}  \cline{6-7} \cline{9-10} \cline{12-13}
&   & Ave. & Std.         & & Ave.      & Std. & & Ave.  & Std. & & Ave.  & Std.\\
\hline
Graph-Convexity&           &$85.10\%$  &$0.08$   & &$57.24\%$  &$0.24$& &$63.00\%$  &$0.21$& &$71.40\%$  &$0.15$\\
RandersGeo&                &$86.71\%$  &$0.10$   & &$78.61\%$  &$0.12$& &$81.10\%$  &$0.10$& &$84.09\%$  &$0.09$\\
RSF&                       &$88.68\%$  &$0.07$   & &$79.60\%$  &$0.13$& &$84.45\%$  &$0.08$& &$89.33\%$  &$0.05$\\
Dubins&                    &$86.03\%$  &$0.08$   & &$75.27\%$  &$0.15$& &$81.01\%$  &$0.10$& &$86.72\%$  &$0.07$ \\
Elastica&                  &$88.44\%$  &$0.08$   & &$81.70\%$  &$0.11$& &$86.73\%$  &$0.07$& &$90.38\%$  &$0.04$\\
RSF-Convexity   &          &$90.23\%$  &$0.06$   & &$82.22\%$  &$0.11$& &$87.24\%$  &$0.06$& &$90.69\%$  &$0.05$ \\
Dubins-Convexity&          &$87.97\%$  &$0.06$   & &$81.82\%$  &$0.09$& &$86.61\%$  &$0.05$& &$89.18\%$  &$0.07$\\
Elastica-Convexity&        &$90.66\%$   &$0.06$   & &$83.32\%$  &$0.10$& &$88.82\%$  &$0.05$& &$92.07\%$  &$0.03$\\
\shline
\end{tabular}
\end{table*}

Fig.~\ref{fig_TRC} illustrates the qualitative comparison of the Graph-Convexity, RSF-Convexity and Elastica-Convexity models on images from the Convexity dataset~\cite{gorelick2016convexity}. The original images and the initial annotations are respectively demonstrated in columns $1$ and $2$. In column $2$, the  cyan dots represent the physical positions $p$ of the source points $\fp=(p,\theta_p)$, and the arrows are positively collinear to the directions $(\cos\theta_p,\sin\theta_p)$.  The foreground and background scribbles are provided by the Convexity dataset. For fair comparison, we add the segment $[z,p]$ (resp. the segment $[p,q]$ such that $q\in\partial\Omega$ subject to  $(p-z)\propto (q-p)$) to the foreground  scribbles (resp. the background scribbles) to initialize the Graph-Convexity model. In column $3$, the contours from the Graph-Convexity model indeed are approximately convex, but fail to accurately delineate the objective boundaries. In contrast, the RSF-Convexity and Elastica-Convexity models are capable of producing smooth and accurate contours, see columns $4$ and $5$. Moreover, the physical projections of the geodesic paths from both  the RSF-Convexity and Elastica-Convexity models are quite close to each other, despite their different regularities on the curvature.

Fig.~\ref{fig_CompGeos} depicts the comparison results of the introduced RSF-Convexity model against the RandersGeo  model and the classical RSF model on images sampled from the Convexity dataset.  The  annotations shown in column $2$ are exploited to initialize these models. It is known that the RandersGeo model heavily depends on image features made up of image appearance model and the image gradients, as formulated in Section~\ref{subsec_RegionRanders}. Accordingly, when these image features are unreliable for defining the target boundaries, unsatisfactory segmentation results may be yielded, as shown in column $3$ of Fig.~\ref{fig_CompGeos}. This is also the case for the classical RSF model whose velocity function~\eqref{eq_Velocity} are derived using both of the region-based homogeneity measure and image gradient-based features, though the use of the curvature regularity may  increase the local smoothness of the associated planar curves, as depicted in column $4$. By the RSF-Convexity model,  smooth and accurate segmentation contours are constructed  for segmentation in all the three test images, as shown in column $5$, due to the integration of the convexity shape prior and the curvature regularization. 

 The comparison results in Fig.~\ref{fig_Synthetic_Weizzman} are derived from the Graph-Convexity model, the RandersGeo model and the classical EM elastica geodesic model and the Elastica-Convexity model, respectively. The test images in rows $1$ to $3$ are sampled from the Weizmann dataset~\cite{alpert2012image}, and in rows $4$ and $5$ are synthetic images. Among the results, the Elastica-Convexity model indeed obtains satisfactory contours for tracking the target boundaries in all the test images. In contrast, we observe that parts of the segmentation contours derived from the other compared models misalign with the target boundary in some extent. Specifically, the goal of the tests for the images of rows $4$ and $5$ is to separate the quasi-elliptic and polygon shapes from the corresponding background regions. We can see that the Graph-Convexity model can extract  approximately convex shapes, but failing to completely recover the target regions. In addition, we can see that the RandersGeo and EM elastica models lead to non-convex segmentation contours, thus missing some parts of the desired boundaries, due to the influence of the gaps. Note that in Figs.~\ref{fig_TRC} to~\ref{fig_Synthetic_Weizzman}, we utilize the GMM-based  image appearance model for computing the regional homogeneity features, see~\cref{subsec_RegionRanders}. 

\subsection{Quantitative  Comparisons}
In this section, we present the quantitative evaluation results of different models on both synthetic and real images, which are conducted using different initial annotations. In each test, we apply the Jaccard score to measure the accuracy of the segmented regions $\RS$ generated from the tested models. More specifically, the Jaccard score $\JS(\RS,\RG)$ estimates the quality that the segmented region $\RS$ recovers the ground truth region $\RG$ 
\begin{equation}
\JS(\RS,\RG):=\frac{\#(\RS\cap\RG)}{\#(\RS\cup\RG)},
\end{equation}
where $\#(\RS\cap\RG)$ denotes the area of the region $\RS\cap\RG$.

Table~\ref{table_Syntheic} presents the quantitative comparison results of the convexity-constrained geodesic models  against the Graph-Convexity model, the RandersGeo model, and the classical curvature-penalized geodesic models, conducted on the images shown in Fig.~\ref{fig_Synthetic_Weizzman}. For each test image, we sample $10$ planar point distributed evenly along the ground truth boundary, and each point is manually assigned a tangent vector, thus yielding $10$ groups of annotations in conjunction with the scribbles as illustrated in the first column of Fig.~\ref{fig_Synthetic_Weizzman}.  By running the evaluated models  $10$ times per image upon those annotations, the average and standard deviation of the Jaccard scores are demonstrated in Table~\ref{table_Syntheic}. In columns $3$ to $6$ of Table~\ref{table_Syntheic},  the  RSF-Convexity and Elastica-Convexity models obtain higher average values of Jaccard scores than the Graph-Convexity,  RandersGeo and classical curvature-penalized models, due to the benefits from the integration of the convexity shape prior and the curvature penalization. Moreover,  the Dubins-Convexity model achieves comparable performance with the RandersGeo model and the classical curvature-penalized models as shown in columns $2$ to $6$.

The performance of the proposed geodesic models with convexity shape prior  is indeed promising, as shown in Table~\ref{table_Syntheic}, but insufficient to fully access the ability of the proposed models in finding suitable segmentation contours under various initial annotations. For that purpose, we perform the quantitative comparison of the proposed convexity-constrained geodesic models to  the Graph-Convexity model, the RandersGeo model and the classical curvature-penalized geodesic models on the Convexity dataset~\cite{gorelick2016convexity}, under different initial annotations, where the comparison results are presented in Table~\ref{table_HybridAnnotations}.  
The evaluation results taking scribbles as initialization for the evaluated models are illustrated in the second column of Table~\ref{table_HybridAnnotations}. The scribbles used are obtained from the Convexity dataset, augmented with the additional scribbles generated from the segments $\ray\cap\Omega$, where $p$ is the physical position of the source point $\fp$. In contrast with the Graph-Convexity method, the exterior scribbles from the Convexity dataset are \emph{not} used for the evaluated geodesic models. For each test image, we sample $5$ source points $\fp$ whose physical positions $p$ are evenly distributed at the ground truth contour. Each source point and the scribbles are taken as the initialization to set up a run for each tested model. We take the average of the obtained $5$ Jaccard scores, which is referred to as~\emph{mean accuracy score}, for the respective test image. The Ave. and Std. values in this table represent the average and standard deviation of the mean accuracy scores of all images in the Convexity dataset. 

In the landmark points-based tests, we randomly sample $5$ groups of planar points from the ground truth boundary of each test image, and each group is made up of $5$ candidate landmark points. In addition, we take one point as the physical position $p$ of the source point $\fp=(p,\theta_p)$, and manually assign to each $p$ a suitable value $\theta_p$ admitting the edge tangent at $p$. In addition, we choose $2$ (resp. $3$ or $4$) points from the remaining planar points of each sampled group for constructing the $3$ (resp. $4$ or $5$) landmark points-based annotations. Thus we perform $5$ runs for each test models, yielding $5$ Jaccard scores per image. The mean accuracy scores of the obtained Jaccard Scores are shown in columns $3$ to $5$ of Table~\ref{table_HybridAnnotations}. One can see that  the proposed convexity-constrained geodesic models commonly achieve higher mean accuracy scores than the corresponding classical curvature-penalized models without convexity shape prior. Furthermore, the RSF-Convexity and Elastica-Convexity models obtain higher mean accuracy scores than the compared Graph-Convexity model, the RandersGeo model and the classical curvature-penalized models, in terms of both types of annotations.

We perform quantitative comparisons for different geodesic models on $43$ CT images sampled from a dataset~\cite{spencer2019parameter}, where the targets are approximately convex. In order to show the advantages of using convexity constraint, we add Gaussian noise and gaps to those images, see Fig.~\ref{fig_ExampleCT} as an example.  Each model is performed $5$ run per test image, where the initialization is made up of a source point $\fp=(p,\theta)$ and the point $z$ that is the barycentre point of the ground truth region. We randomly choose $5$ physical positions $p$ from the target boundary and the angular components $\theta_p$ is manually set as the counter-clockwise tangent of the boundary at $p$. The evaluated results of different geodesic models are illustrated in Fig.~\ref{fig_Boxplot_CT},  which are exhibited by the boxplots of the mean accuracy scores. In this figure, we observe that the convexity-constrained geodesic models indeed outperform the RandersGeo, RSF and Dubins models, and achieve slightly better performance compared to the classical EM elastica model, which proves the efficiency, accuracy and robustness of our models against the presence of noise and gaps.

\begin{figure}[t]
\centering{
\includegraphics[width=4cm]{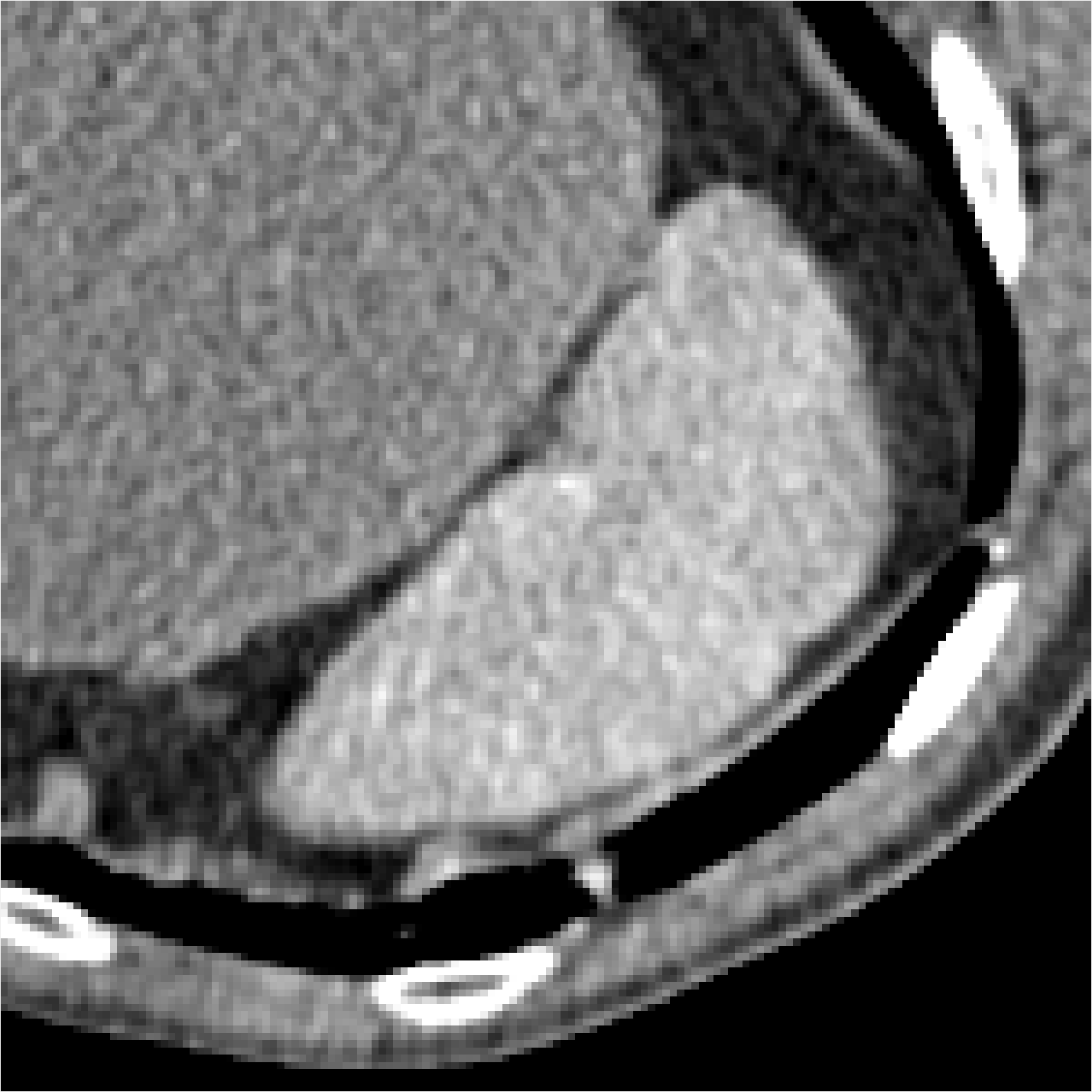}~\includegraphics[width=4cm]{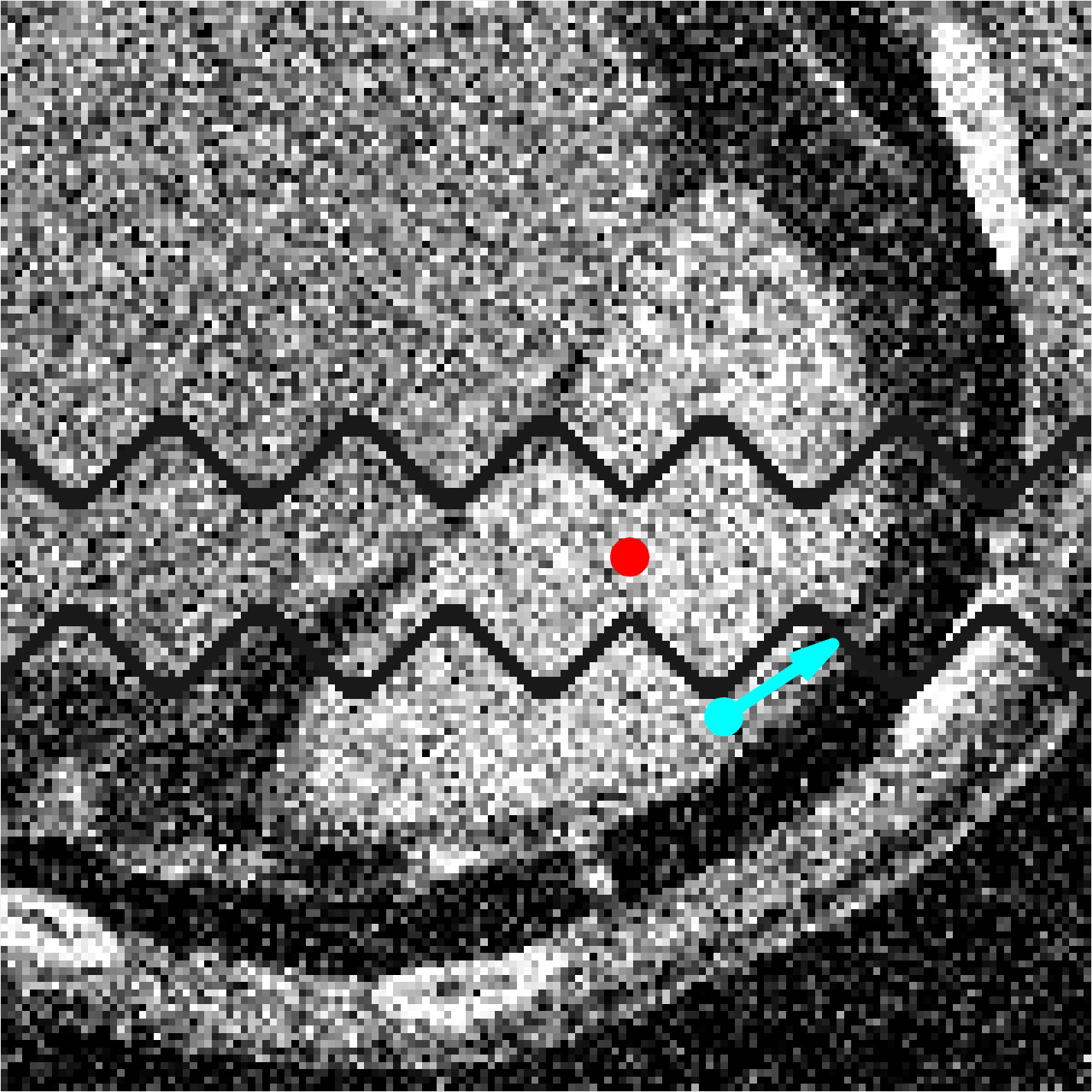}}
\caption{\textbf{Left} An original clean CT image.  \textbf{Right} The image blurred with Gaussian noise and gaps. The red dot indicates the point $z$ and the cyan dot with the arrow denotes a source point $\fp$.}
\label{fig_ExampleCT}				
\end{figure}

\subsection{Discussion on Computation Times}
The computation complexity in our models is dominated by the computation of the GMM-based velocity and the running of the HFM method during curve evolution. We focus on the analysis of the HFM method, since the velocity computation is embarrassingly parallel and could thus easily be accelerated by GPU programming.
The computation complexity $\cO(I K  N \ln N)$ of the HFM method depends on the number $N$ of discretization points of the grid, and the number $IK$ of points in the stencil~\eqref{eqdef:stencil}. Practical run times also modestly depend on the specific parameters and the test cases, such as the relaxation parameter $\ve>0$ used for the curvature penalized models \cite{mirebeau2018fast}, the profile of the cost function and walls, and more importantly the use of a narrow band which constrains the front propagation to a small subregion.
The experiments in this paper were conducted on an Intel Core i9 $3.6$GHz architecture with $96$GB RAM, using a C++ implementation. We take the image shown in Fig.~\ref{fig_Demo} as an example to report the execution time, where the size of the grid is $346\times 599\times 60$.  The HFM method  associated to the Elastica-Convexity model requires around $38$ seconds for tracking the geodesic path  shown  in Fig.~\ref{fig_Demo}d. No narrow band was considered in this test, which relied on a gradients-based velocity.  
A GPU implementation of the eikonal solver~\cite{mirebeau2021massively} led to a strong acceleration, and to running times compatible with user interaction. In the same example of Fig.~\ref{fig_Demo}d, computation time is reduced to $2.5$ seconds on a laptop equipped with an Nvidia 2060 Max-Q GPU. 

\begin{figure}[t]
\centering
\includegraphics[width=8.2cm]{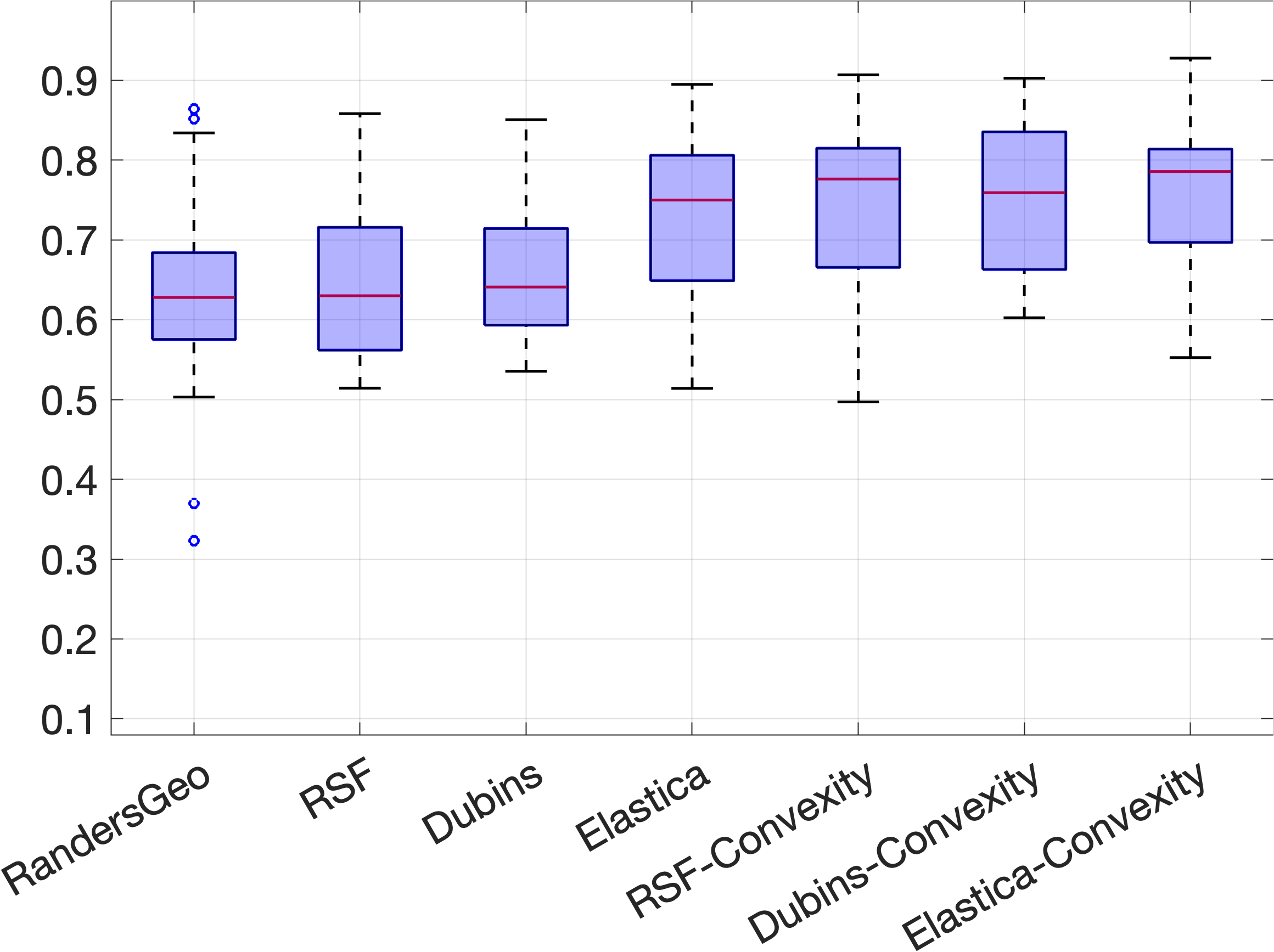}
\caption{Box plots of the mean accuracy scores on $43$ CT images blurred by Gaussian noise and artificial gaps for different geodesic models. }
\label{fig_Boxplot_CT}				
\end{figure}

\section{Conclusion}
\label{sec_Conclusion}

In this paper, we introduce new geodesic models by imposing a convexity shape constraint to the Reeds-Sheep forward,  Dubins car and Euler-Mumford elastica models. The physical projections of  geodesic paths computed from these new models are simple closed and convex planar curves, which are then applied as an efficient solution to the convex shape-constrained active contour problem. In summary, we addressed three crucial issues in order to generate such geodesic paths, including (i) the construction of new geodesic metrics and the associated Hamiltonians, which admit the particular constraint from the convexity shape prior,  (ii) the design of the set collecting all admissible orientation-lifted curves whose physical projections admit the closedness and simplicity restriction, and (iii) the algorithms of applying the proposed geodesic models for interactive image segmentation, such that the user intervention and image  features can be efficiently incorporated into the extraction of geodesic paths and boundary contours. 
%Experimental results on both synthetic and real images illustrated the advantages of the proposed geodesic models.

\section*{Acknowledgement}
This work is in part supported by the National Natural Science Foundation of China (NO.~61902224), the Shandong Provincial Natural Science Foundation (NO.~ZR2022YQ64), a project in QLUT (NO.~2022PY11), and the French government under management of Agence Nationale de la Recherche as part of the ``Investissements d'avenir'' program, reference ANR-19-P3IA-0001 (PRAIRIE 3IA Institute). Tai is supported by GRF project HKBU-12300819, NSF/RGC Grant N-HKBU214-19 and RC-FNRA-IG/19-20/SCI/01. Shu is supported by the Young Taishan Scholars (NO. tsqn201909137).
% or


\appendix[Relaxation of the RSF and elastica metrics]
Two numerical approaches can be considered for the numerical computation of curvature penalized shortest paths: either approximate the Hamiltonian of the models, which leads to the HFM finite difference scheme framework \cite{mirebeau2019hamiltonian}, or approximate the metric, which leads to semi-Lagrangian schemes considered in the earlier works \cite{duits2018optimal,chen2017global}. The latter approach, briefly recalled below, is more geometric and intuitive, but leads to numerical difficulties: in order to use the efficient fast marching numerical solver, one needs to design stencils obeying a geometric acuteness property depending on the (approximate) metric \cite{sethian2003ordered}, which is difficult for models featuring a strong and generic anisotropy, in particular the Elastica-Convexity and Dubins-Convexity models considered in this paper. For this reason, the HFM approach was preferred.\\
\noindent\textbf{The RSF Model}.
The classical RSF model invokes the metric~\eqref{eq_RSFMetric} to construct the energy of  smooth curves. In order to implement the fast marching method as the numerical solver for the associated eikonal equation, Duits~\emph{et al.}~\cite{duits2018optimal} proposed a new relaxed RSF metric of strong anisotropy to approximate the original RSF metric $\cF^{\rm RS}$. Letting  $\varepsilon\in(0,1)$ be a relaxation parameter, the relaxed RSF metric $\cF^{\rm RS}_\varepsilon:\bM\times\bE\to[0,+\infty]$ has a form of
\begin{equation*}
\cF^{\rm RS}_\varepsilon(\fx,\dot\fx)=\sqrt{\langle \dx,\dfnt\rangle^2_++\beta^2|\dot\theta|^2+
 \varepsilon^{-2}(\|\dx\|^2-\langle\dx, \dfnt\rangle^2_+)},
\end{equation*}
for any point $\fx=(x,\theta)$ and any vector $\dot\fx=(\dx,\dot\theta)$. 
Similarly, the RSF-Convexity metric may be approximated by introducing the additional term $\ve^{-2} (\theta_-)^2$ under the square root. 
However, obtaining similar and numerically tractable approximations of the Elastica-Convexity and Dubins-Convexity metrics is an open question.

\noindent\textbf{The EM Elastica Model}. Chen~\emph{et al.}~\cite{chen2017global} introduced a  Finsler metric with a Randers form for approximately computing optimal elastica curves. This particular Randers metric is established over the orientation-lifted space $\bM$ and can be expressed as
\begin{equation*}	
\cF^{\rm EM}_\varepsilon(\fx,\dot\fx)=\sqrt{\varepsilon^{-2}\|\dx\|^2+2\varepsilon^{-1}\beta^2|\dot\theta|^2}-(\varepsilon^{-1}-1)\langle \dx,\dfnt\rangle.
\end{equation*}
In~\cite[Appendix 2]{chen2017global}, the authors prove that the geodesic paths derived from the metric $\cF^{\rm EM}_\varepsilon$ uniformly converge to those associated to the original elastica metric $\cF^{\rm EM}$, see~Eq.~\eqref{eq_FEMMetric}, as $\varepsilon\to 0$. The convergence results was adapted and further generalized from the proof of those in the RSF model~\cite[Appendix A]{duits2018optimal}.

%\appendices

%% you can choose not to have a title for an appendix
%% if you want by leaving the argument blank
%\section{}

%% use section* for acknowledgment
%\ifCLASSOPTIONcompsoc
%  % The Computer Society usually uses the plural form
%\else
%  % regular IEEE prefers the singular form
%\fi

\ifCLASSOPTIONcaptionsoff
  \newpage
\fi



\bibliographystyle{IEEEtran}
\bibliography{minimalPaths}

%\begin{thebibliography}{1}
%\end{thebibliography}


%\end{IEEEbiography}

%\newpage


\end{document}